%% file: main.tex
\documentclass[letterpaper]{article}
\usepackage{aaai25}  
\usepackage{times}  
\usepackage{helvet}  
\usepackage{courier}  
\usepackage[hyphens]{url}  
\usepackage{graphicx} 
\usepackage{natbib}  
\usepackage{caption} 
\usepackage{algorithm}
\usepackage{algorithmic}

\usepackage{amsthm}
\usepackage{amsmath}
\usepackage{amsfonts}
\usepackage{mathtools}
\usepackage{amssymb}
\usepackage{enumitem}
\usepackage{relsize}
\usepackage{xcolor}
\usepackage{bm}
\usepackage{tikz}
\usepackage{subcaption}
\usepackage{array}
\usepackage{newfloat}
\usepackage{listings}
\DeclareCaptionStyle{ruled}{labelfont=normalfont,labelsep=colon,strut=off} 
\floatstyle{ruled}
\newfloat{listing}{tb}{lst}{}
\floatname{listing}{Listing}

\title{Designing Ambiguity Sets for Distributionally Robust Optimization Using Structural Causal Optimal Transport}
\author{
    Ahmad-Reza Ehyaei \textsuperscript{\rm 1},
    Golnoosh Farnadi \textsuperscript{\rm 2}
    Samira Samadi \textsuperscript{\rm 1}
}
\affiliations {
    \textsuperscript{\rm 1} Max Planck Institute for Intelligent Systems Tübingen,
Tübingen, Germany\\
    \textsuperscript{\rm 2} Mila - Québec AI Institute, Université de Montréal  \\
  Montréal, Canada\\
    ahmadreza.ehyaei@uni-tuebingen.de, 
    farnadig@mila.quebec,
    ssamadi@tuebingen.mpg.de
}

\input{commands}

\usepackage{bibentry}

\begin{document}

\maketitle


\begin{abstract}
Distributionally robust optimization tackles out-of-sample issues like overfitting and distribution shifts by adopting an adversarial approach over a range of possible data distributions, known as the ambiguity set. To balance conservatism and accuracy, these sets must include realistic probability distributions by leveraging information from the nominal distribution.
Assuming that nominal distributions arise from a structural causal model with a directed acyclic graph $\mathcal{G}$ and structural equations, previous methods such as adapted and $\mathcal{G}$-causal optimal transport have only utilized causal graph information in designing ambiguity sets. 
In this work, we propose incorporating structural equations, which include causal graph information, to enhance ambiguity sets, resulting in more realistic distributions. We introduce structural causal optimal transport and its associated ambiguity set, demonstrating their advantages and connections to previous methods.
A key benefit of our approach is a relaxed version, where a regularization term replaces the complex causal constraints, enabling an efficient algorithm via difference-of-convex programming to solve structural causal optimal transport. We also show that when structural information is absent and must be estimated, our approach remains effective and provides finite sample guarantees.
Lastly, we address the radius of ambiguity sets, illustrating how our method overcomes the curse of dimensionality in optimal transport problems, achieving faster shrinkage with dimension-free order.
\end{abstract}

\section{Introduction}
\label{sec:intro}
Distributionally Robust Optimization (DRO) is a data-driven framework designed to address out-of-sample challenges, such as distribution overfitting and distributional shifts, by minimizing potential discrepancies between in-sample expected loss and out-of-sample expected loss. DRO achieves this by defining a distributional ambiguity set (DAS) that encompasses a range of possible data distributions around the estimated true probability, ensuring that the DAS contains the unknown true underlying distribution with certainty or at least with high confidence. To guarantee the model's performance over out-of-sample distributions, DRO employs an adversarial approach that minimizes the worst-case loss to identify the optimal model~\cite{blanchet2024distributionally}.
\begin{figure}
    \centering
    \includegraphics[width=0.6\linewidth]{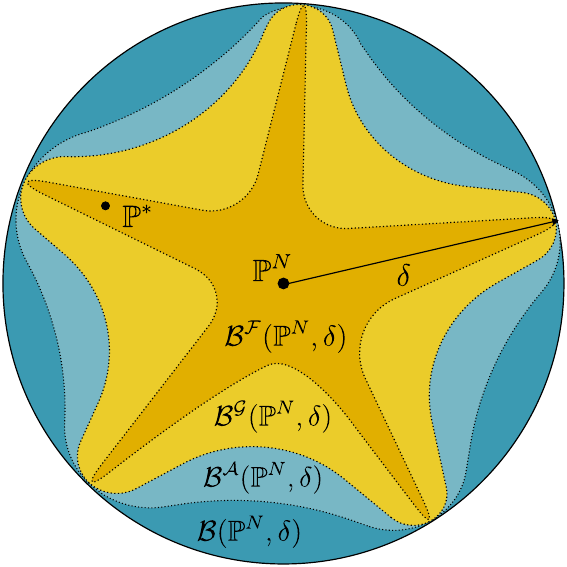}
    \caption{The underlying distribution $\bP^\ast$ originates from a causal structure. The dark blue region represents $\BPD$ the ambiguity set for the classical OT. The light blue region corresponds to $\BAPD$ the DAS for adapted optimal transport. The light yellow region denotes $\BGPD$ the DAS for $\cG$-causal OT, and the dark yellow region represents $\BFPD$ the DAS for our structural causal OT with diameter $\delta$.\\}
    \label{fig:ots}
\end{figure}
Ambiguity sets are typically categorized into two groups: discrepancy-based and moment-based. Discrepancy-based sets include distributions close to a nominal distribution according to a discrepancy measure, while moment-based sets include distributions whose moments satisfy certain properties \cite{rahimian2022frameworks}. Among discrepancy-based sets, the Wasserstein distance is often preferred, as it quantifies the discrepancy by the minimal transportation cost, ensuring computational tractability through strong dual formulations and convergence guarantees. Additionally, Wasserstein ambiguity sets are robust against outliers and can handle both continuous and discrete distributions, unlike the KL divergence ball, which is limited to discrete distributions~\cite{guo2017ambiguity}.

\newcolumntype{C}[1]{>{\centering\arraybackslash}m{#1}}
\begin{table*}[t]  
    \centering
\begin{tabular}{m{4.5cm} m{4.8cm} C{3.5cm} C{2.5cm}}
\textbf{Optimal Transport Variant} & \textbf{Information Utilization} & \textbf{Wasserstein Distance} & \textbf{Ambiguity Set} \\ \hline
\rule{0pt}{3ex}
Classical~\cite{villani2009optimal,peyre2017computational,ambrosio2021lectures}  & No Constraints & $\WPQ$ & $\BPD$ \\ 
Adopted~\cite{backhoff2017causal, lassalle2018causal, xu2020cot} & Preserves Causal Order & $\WAPQ$ & $\BAPD$ \\ 
$\cG$-Causal~\cite{cheridito2023optimal} & Preserves Causal Graph & $\WGPQ$ & $\BGPD$ \\ \hline\rule{0pt}{3ex}
\textbf{Structural Causal} & Preserves Structural Equations & $\WFPQ$ & $\BFPD$ \\ 
\textbf{Relaxed Structural Causal} & Partially Preserves Structural Equations with Penalty Term & $\WRPQ$ & $\BRPD$ \\ \hline
\end{tabular}
\caption{
Comparison of optimal transport variants using nominal distribution information to design ambiguity sets, with corresponding notations for Wasserstein distance and ambiguity set diameter $\delta$ For two probability distributions $\bP,\bQ \in \PX$.\\
}
\label{tab:varaints}
\end{table*}

\noindent In designing the DAS two points need to be considered~\cite{rahimian2022frameworks}:
\begin{itemize}
\item[\textbf{P1.}]  What distributional information should the DAS include?
\item[\textbf{P2.}]  How large should the radius of the DAS?
\end{itemize}
To understand the significance of \textbf{P1}, consider the corresponding DAS of classical Wasserstein DRO. It includes all probability distributions within a specified Wasserstein distance from the empirical distribution, usually based on a metric like the $\ell_p$ norm. While this approach works well for unstructured data, it fails for data with special structures, such as temporal patterns or causal relationships. In such cases, the Wasserstein DAS should be refined to exclude unrealistic scenarios. Without this refinement, models become overly conservative, reducing accuracy.

Previous works address this issue by introducing optimal transport (OT) with additional constraints using partial distribution information. \textit{Adapted OT} emphasizes the temporal structure of features~\cite{backhoff2017causal}. In causal structures, features are arranged according to a directed acyclic graph (DAG) $\cG$. While adapted OT preserves the feature hierarchy, it fails to capture causal models. To improve this, \cite{cheridito2023optimal} introduce Wasserstein distances for causal models that respect the DAG $\cG$. The resulting $\cG$-causal OT’s DAS is a subset of the adapted DAS, considering all structural causal models with graph $\cG$ as candidates. Although this set is narrower than the classical DAS, it still includes unrealistic scenarios by overlooking functional relationships between features. For instance, if a variable $\X_i$ weakly affects its descendant $\X_j$, this weak relationship is not considered in the $\cG$-causal OT.

\paragraph{Our Contribution.}
To tackle this challenge, we propose a novel variant of OT that takes into account the structural equations of causal models. This approach enables us to refine the DAS by limiting it to probability distributions derived from structural causal models with identical structural equations but potentially varying noise variables. This concept facilitates the connection between endogenous and exogenous spaces in constructing the DAS. The resulting duality allows us to define the DAS within the exogenous space, where variable independence can be assumed, simplifying the design process. We can subsequently map back to the feature space to craft our preferred DAS.

Another advantage of our approach is the flexibility to define a relaxed version by replacing causal constraints with an entropy regularization term. This allows the implementation of an efficient algorithm combining difference-of-convex (DC) programming and Sinkhorn's method to solve structural causal OT, a capability not available in $\cG$-causal OT.

Since the design of the ambiguity set relies on structural equations, we demonstrate that, in real-world applications where the structural equations are estimated, the solution of the structural causal OT converges to the true solution. This offers a convergence guarantee for finite sample scenarios.

Moreover, we address \textbf{P2} by determining the optimal radius necessary to ensure that the true probability distribution is included within the DAS. Our approach mitigates the curse of dimensionality commonly encountered in data-driven ambiguity set design for many OT problems. 
In the numerical study, we demonstrate the impact of our method and its properties. Additional theoretical results, algorithm descriptions, and numerical outcomes are provided in the appendix, while proofs of our assertions are included in the supplementary material.
\subsection{Related Work}
There are various methods that address how to design DAS, how to add constraints to achieve desirable OT, and discuss the magnitude of DAS. In Tab.~\ref{tab:varaints} and Fig.~\ref{fig:ots}, the main methods, along with the information used in designing DAS and their relationships, are summarized. These methods include:

\noindent \textbf{Discrepancy-based OT.}
Various discrepancy-based methods exist, such as $\phi$-divergences~\cite{love2015phi, lam2016robust, duchi2021learning} and goodness-of-fit tests~\cite{bertsimas2018data}. However, we focus on the Wasserstein OT due to its advantages over other discrepancies~\cite{gao2017wasserstein, mohajerin2018data, blanchet2019robust}.

\noindent \textbf{Temporal OT.}
Temporal OT~\cite{backhoff2017causal, lassalle2018causal, xu2020cot, bartl2021wasserstein, backhoff2022stability} mainly addresses OT for stochastic processes by preserving temporal structure. However, this approach does not encompass all information about the distribution derived from the causal structure.
In~\cite{eckstein2024computational}, Sinkhorn's algorithms for adapted OT are proposed by introducing a relaxed version. 

\noindent \textbf{$\cG$-Causal OT.}
In \cite{cheridito2023optimal}, the authors propose $\cG$-causal OT to preserve the causal graph $\cG$ in the OT. However, this work does not provide a computational method for estimating $\cG$-causal solutions.

\noindent \textbf{Structured ambiguity Set.}
In the case where all features are independent, our work relates to factored multi-marginal OT~\cite{tran2021factored} and CO-OT~\cite{titouan2020co}. In this scenario, the ambiguity set also connects to the ambiguity hyperrectangle~\cite{chaouach2022uncertain}.

\noindent \textbf{Diameter of Ambiguity Set.}
For the estimation of $W(\bP,\bP^N)$ for finite samples, various works exist \cite{bolley2007quantitative, fournier2015rate, dedecker2019behavior}. The work \cite{WeedShaprp2019} provides the sharp order. Additionally, \cite{chaouach2022uncertain, chaouach2023comparing} discuss the diameter of the ambiguity set under the independent condition of features.
\section{Preliminary Knowledge}
\label{sec:background}

\noindent\textbf{Data Model.}
Let $\X = (\X_1, \dots, \X_n) \in \cX_1 \times \dots \times \cX_n= \cX$ denote the $n$-dimensional features, the set of $N$ observations $\{x^i\}_{i = 1}^{N}$ used to construct the empirical distribution $\bP^N$, defined as $\bP^N \coloneqq \frac{1}{N} \sum_{i=1}^{N} \delta_{x^i}$, where $\delta_{x}$ is the Dirac delta function. 

Assume the feature space is modeled by a \textit{structural causal model (SCM)} $\cM = \langle \mathcal{G}, \X,\cF, \mathbf{U}, \bP_\U \rangle$~\cite{pearl2009causality}. This includes \textbf{structural equations} $\mathcal{F} = \{f_1, f_2, \dots, f_n\}$  where $\{\X_i := f_i(\X_{\text{Pa}(i)}, \mathbf{U}_i)\}_{i=1}^{n}$, which describe the causal relationships between an endogenous variable $\X_i$, its causal predecessors $\X_{\text{Pa}(i)}$, and an exogenous variable $\mathbf{U}_i$ representing unobservable factors in space $\cU_i$, and $\U = (\U_1,\dots, \U_n)$ lives on the whole exogenous space is $\cU = \cU_1 \times \dots \times \cU_n$. The causal relations are represented by a directed acyclic \textbf{causal graph} $\mathcal{G}$. 

A DAG imposes a \textbf{causal order} (topological order), refers to the sequence in which variables can be arranged such that each variable is only affected by the variables preceding it in the order~\cite{pearl2009causality,peters2017elements}.
Causal order Only specifies the sequence of variables, indicating which variables can potentially affect others, but not the detailed nature of those effects.

By \textit{causal sufficiency} and \textit{no hidden confounders}, we can suppose exogenous variables to be mutually independent, allowing $\bP_{\U}$ to be written as $\prod_{i=1}^{n}\bP_{\U_i}$~\cite{peters2017elements}.

Since perturbations in SCMs are utilized by counterfactuals, it is necessary for SCMs to be counterfactually identifiable to ensure that counterfactuals can be learned from sample data. One prominent family of counterfactually identifiable models is the \emph{Bijective Generation Mechanism (BGM)}~\cite{nasr2023counterfactual}. 
In BGM, the structural equations $\cF$ have a \textbf{reduced-form mapping} $g: \cU \to \cV$, where $\X$ can be expressed as a bijective function of the exogenous space, i.e., $\X = g(\U)$. This bijective ensures no information is lost from exogenous to endogenous variables. 

An important example of BGM is the \textbf{Additive Noise Models (ANM)}, where structural equations are given as:
\begin{align*}
\{\X_i \coloneqq & f_i(\X_{\pa(i)}) + \U_i \}_{i=1}^{n} 
 \implies  \nonumber\\ &  \U = (I - f)(\X) \implies  
 \quad \X = (I - f)^{-1}(\U) \label{eq:ANM}
\end{align*}
As seen in the above equation, the reduced-form mapping is $g = (I - f)^{-1}$, where $I(x) = x$ is the identity function. 
ANM is often preferred over general SCMs due to its simplicity, interpretability, and effective handling of noise, making it ideal for fields such as statistics, causal inference, signal processing, image processing, economics, and social sciences, where additive noise is prevalent. In this work, we focus on the ANM, but our results are extendable to BGMs.

\paragraph{Structured Ambiguity Set.}
In the variants of Wasserstein OT, the ambiguity set is typically defined through coupling.
Let $\bP,\bQ \in \PX$ be probability distributions; the distribution $\pi \in \cP(\XTX)$ is called a \textit{coupling} or \textbf{plan} if for all measurable subsets $A,B \subset \cX$, $\pi(A \times \cX) = \bP(A)$ and $\pi(\cX \times B) = \bQ(B)$. Let $\PPQ$ represent the set of all couplings between $\bP$, and $\bQ$. Each $\pi \in \PPQ$ shows how $\bP$ transforms into $\bQ$. All plans starting from $\bP$ are denoted by $\Pi(\bP, \ast) \coloneqq \bigcup_{\bQ \in \PX} \PPQ$.

The structured DAS for $\bP$ is a subset of plans $\Pi(\bP, \ast)$ that meet specific constraints $\cC = \{c_k\}$, typically defined by parameters $\bdelta = \{\delta_k\}$:
\begin{equation*}
\Pi^\cC(\bP,\bdelta) \coloneqq \{\pi \in \PPS :   c_k(\pi,\delta_k), \forall c_k \in \cC \}    
\end{equation*}
The desirable property of $\PCPQ$ is its closeness under the weak topology in probability space, which guarantees the existence of solutions in OT theorems.
The corresponding ambiguity set for $\Pi^\cC(\bP,\bdelta)$ is derived by finding the marginal distribution of each plan on the second coordinate ($\marg_2$):
\begin{equation}
\label{eq:BCPD}
 \cB^\cC(\bP, \bdelta) = \{\marg_2(\pi) : \pi \in \Pi^\cC(\bP,\bdelta)\}   .
\end{equation}
In DRO, the worst-case loss is obtained by taking the expectation of a given function $\psi: \cX \rightarrow \bR$ over the ambiguity set. The following alternative formulation, by definition, often facilitates computations.
\begin{equation}
\label{eq:worstloss}
    \sup_{\bQ \in \cB^\cC(\bP,\bdelta)} \bigg\{\expt{y\sim \bQ}[\psi(y)]\bigg\} = \sup_{\pi \in \Pi^\cC(\bP,\bdelta)} \left\{\expt{(.,y)\sim \pi}[\psi(y)]\right\}
\end{equation}

Similar to constrained plans, if there are constraints on the space of probability measures, the corresponding subset is denoted by $\PCX$.
Let $c(\cdot,\cdot): \cX \times \cX \rightarrow [0, \infty]$ be a transportation cost function that is non-negative and upper-semi-continuous. For each $p \in [1, \infty]$, the set of $p$-integrable distributions $\cP^p(\cX)$ with respect to $c$ is defined as:
\begin{align*} 
\left\{ \bP \in \PX \ \middle|\ \int_{\cX} c(x, x_0)^p \bP(\d x) < \infty \ \text{for some } x_0 \in \cX \right\}
\end{align*}

The $\cC$-Wasserstein distance between $\bP \in \PCX \cap \cP^p(\cX)$ and $\bQ \in \PX \cap \cP^p(\cX)$ is defined by finding the lowest-cost constrained transport plan:
\begin{equation}
\label{eq:WCPQ}
\WCPQ \coloneqq \left(\inf_{\pi \in \PCPQ} \left\{\expt{(x,y)\sim\pi}[c^p(x,y)]\right\}\right)^{\frac{1}{p}} 
\end{equation}
In cases where $\PCPQ = \emptyset$, the $\WCPQ$ is set to $\infty$.

\paragraph{Wasserstein Ambiguity Set.}
In classical OT~\cite{villani2009optimal,peyre2017computational}, transport plans and distributions are unconstrained, except for the transportation cost. For two probability measures $\bP, \bQ \in \PX$, the Wasserstein distance $\WPQ$ represents the optimal cost of transporting one distribution to the other. The corresponding ambiguity set, defined by a perturbation radius $\delta \in \bRP$, is given by:
\begin{equation*}
\PPD = \{\pi \in \PPS: \expt{(x,y)\sim\pi}[c(x,y)] \leq \delta \}
\end{equation*}
By computing the second marginal distribution on $\PPD$, the ambiguity set can be expressed as:
\begin{equation*}
\BPD = \{\bQ \in \PX: \WPQ \leq \delta\}.
\end{equation*}

\paragraph{Adopted Ambiguity Set.}
Adapted OT~\cite{backhoff2017causal} was introduced to address the limitations of classical OT in preserving temporal constraints for two stochastic processes. Given two discrete-time stochastic processes $\bP$ and $\bQ$ with indices $i \in [n]$, the coupling $\pi$ must respect the temporal structure. Consequently, the plans should satisfy the temporal conditional distribution constraints as follows:
\begin{align*}
    &\PAPQ = \big\{\pi \in \PPQ: \ \pi-\text{almost sure}  \ \forall x,y \in \cX, \ \\&  \ 
    \pi(\d y_i \mid \d x_1, \dots, \d x_n) = \pi(\d y_i \mid \d x_1, \dots, \d x_i), \forall \ i \in [n]\big\}
\end{align*}
The adapted ambiguity set and Wasserstein distance are denoted by $\BAPD$ and $\WAPQ$, respectively, similar to Eq.~\ref{eq:BCPD} and Eq.~\ref{eq:WCPQ}.

\paragraph{$\cG$-Compatible Ambiguity Set.}
Since adapted OT cannot preserve complex structures like causal graph dependencies, \cite{cheridito2023optimal} proposed the $\cG$-causal OT framework. A probability $\cP \in \PX$ is called compatible with a sorted causal graph $\cG$ if there exists a random variable $\X \sim \bP$, along with measurable functions $f_i : \cX_{\pa(i)} \times \bR^{d_i} \rightarrow \cX_i$ for $i = 1, \dots, n$, and independent random variables $\U_i$ for all $i \in [n]$ such that:
\begin{align*}
    \X_i = f_i(\X_{\pa(i)}, \U_i) \quad \text{for all } i = 1, \dots, n.
\end{align*}
The set of $\cG$-compatible measures is denoted by $\PGX$. 
A $\cG$-compatible plans, which capture $\cG$ graph, is defined as:
\begin{align}
    &\PGPQ = \big\{\pi \in \PPQ: \ \forall \ i \in [n], \ \pi\text{-almost all } (x,y), 
    \nonumber\\&
    \pi(\d x_1, \d y_1, \dots, \d x_n, \d y_n) = \bigotimes_{i=1}^n \pi(\d x_i, \d y_i \mid x_{\pa(i)}, y_{\pa(i)}) 
    \nonumber\\&
    \text{ and} \quad \pi(\d x_i \mid x_{\pa(i)}, y_{\pa(i)}) = \bP(\d x_i \mid x_{\pa(i)})  \big\}.
    \label{eq:gcom}
\end{align}
The corresponding DAS and Wasserstein metric are denoted by $\BGPD$ and $\WGPQ$.
Since preserving the causal graph inherently includes preserving the causal order of features, $\cG$-compatible plans are a subset of adapted plans.

By reviewing the main definitions and notations, we are prepared to present our method, which addresses the limitations of previous methods in fully capturing the information of causal models.
\section{Structural Causal Ambiguity Sets}
\label{sec:ambiguity}
The adopted $\cG$-causal ambiguity set retains only the causal graph structure without requiring the full details of the causal model. For example, if the nominal distribution indicates a weak relationship between a parent $\X_i$ and its child $\X_j$ (e.g., $\X_i = \alpha \X_j$ with $\alpha \approx 0$), the $\cG$-causal set neglects this weak dependency. To address this limitation, we propose a new OT variant that incorporates structural equations from the SCM, thereby capturing these dependencies and utilizing more information than the causal graph $\cG$ alone.

Before presenting our method, we highlight a natural assumption. Let $c$ be the cost function on the feature space. Given an invertible SCM, there exists a bijective map $g$ such that $x = g(u)$. We define the push-forward cost $\tilde{c} = c \circ (g \times g)$ on the exogenous space as $\tilde{c}(u,u') = c(g(u), g(u'))$. Since the variables in the exogenous space are mutually independent, each $\cU_i \times \cU_i$ can have its own cost function $\tilde{c}_i$, allowing $\tilde{c}$ to be decomposed into components. As a result, $\tilde{c}$ is expected to have a simpler form. In summary, we make the following assumptions.
\begin{assumption}
\label{asm:space}
\begin{enumerate}[label=(\roman*)]
    \item $\cM$ is a ANM, with structural equations $\cF = \{f_i\}$ and $g$ is bijective reduced-form mapping.
    \item The random variables $\U_i$ are independent and take values in the $\cU_i \subseteq \bR^{d_i}$ that is equipped by the norm $\tc_i$.
    \item The push-forward of the cost function $c$ to the exogenous space has the form: $$\tc(u,u') = \left(\sum_{i=1}^n \tc_i(u_i, u'_i)^p\right)^{\frac{1}{p}}, \forall u_i, u'_i \in \cU_i \text{ and } p\geq 1.$$
\end{enumerate}
\end{assumption}
Now, we are prepared to present our constraints in both probability space and plans.
\begin{definition}[\textbf{$\cF$-Compatible Measures}]
\label{def:Fcomp}
A measure is compatible with the structural equations $\cF$ if its $g^{-1}$ push-forward distribution over the exogenous space is factored,
\begin{equation*}
\label{def:fcomp}
   \PFX = \left\{\bP \in \PX: g^{-1}_{\#}\bP = \BOIN \tP_i, \quad \tP_i \in \cP(\cU_i)\right\},
\end{equation*}
where $\bigotimes$ means the product of measures.
\end{definition}
Another useful definition of $\PFX$ is that it includes all $g$-pushforward distributions of $\BOIN \tP_i$ where $\tP_i \in \cP(\cU_i)$. This duality facilitates conversion between spaces, simplifying our results. The following lemma outlines the properties of $\cF$-compatible measures.
\begin{proposition}
\label{prp:SCP}
Let $\bP \in \PFX$, then:
\begin{itemize}
\item[(i)] There exists a random variable $\X \sim \bP$ along with independent random variables $\U_1, \dots, \U_n$ in the space $\cU_i$ such that,
\begin{equation*}
\X_i = f_i(\X_{\pa(i)}, \U_i) \quad \text{for all $i =1 , \dots, n$}.
\end{equation*}

\item[(ii)] 
The measure $\bP$ can be decomposed as
\begin{equation*}	
\bP(\d x_1, \dots, \d x_n) = \bigotimes_{i=1}^n \bP \brak{\d x_i \mid x_{\pa(i)}},
\end{equation*}
which means variables are conditionally independent of their non-descendants given their parents.
\end{itemize}
\end{proposition}

We are now ready to address \textbf{P1} in our method, which determines the specific information that should be considered in the design of the ambiguity set.

\begin{definition}[\textbf{$\cF$-Compatible Plans}]
\label{def:fplan}
The plan $\pi$ is called $\cF$-compatible if its pushforward map $\tpi = \gmgmpush \pi$ under $\gm \times \gm$ is factored in the exogenous space as follows:
\begin{align*}
    &\PFPQ = \bigg\{\pi \in \PPQ: \text{for } \tpi\text{-almost sure and all } 
    \\&
     (u,w) \in \cU \times \cU, \gmgmpush\pi(\d u, \d w) = \bigotimes_{i=1}^n \tpi_i(\d u_i, \d w_i),
    \\&
    \text{such that } \tpi_i \in \Pi\left(\marg_i(\gms \bP), \marg_i(\gms \bQ)\right)
\bigg\}.
\end{align*}
where $\marg_i$ is the marginal distribution over the coordinate $\cX_i$.  
\end{definition}

The intuition behind this definition is straightforward: we consider the plans $\pi$ whose push-forward in the exogenous space decomposes onto $\cU_i \times \cU_i$, as we have mutually independent noise by the assumption. Now we investigate the distributional properties implied by the definition of $\cF$-compatible plans.

\begin{proposition}
\label{prp:opx_prop}
Let $\bP, \bQ \in \PFX$ and $\pi \in \PFPQ$, then we have:

(i) if $(X,Y)\sim \pi$ then there exists measurable functions
\begin{align*}
h_i \colon \cX_{i} \times \cX_{\pa(i)} \times \cX_{\pa(i)} \times \mathbb{R}^{d_i} \rightarrow \cX_i,
\quad i \in [n],
\end{align*}
and $\bR$-valued random variables $\V_1, \dots, \V_n$ 
such that $\X,\V_1, \dots, \V_n$ are mutually independent and
\begin{equation*}
\Y_i = h_i(\X_{i}, \X_{\pa(i)}, \Y_{\pa(i)}, \V_i) \quad \mbox{for all } i \in [n].
\end{equation*}

(ii) for all $i = 1, \dots, n$ and $\pi$-almost all $(x, y) \in \cX \times \cY$ 
\begin{align*}
&\pi(\d x_1, \d y_1, \dots, \d x_n, \d y_n) = \bigotimes_{i=1}^n \pi(\d x_i, \d y_i \mid 
x_{\pa(i)}, y_{\pa(i)}) 
\\&
\text{and} \quad \pi(\d x_i, \d y_i \mid x_{\pa(i)}, y_{\pa(i)}) \in \Pi_i, \quad \forall i \in [n],
\end{align*}
where $\Pi_i = \Pi\left(\bP(\d x_i \mid x_{\pa(i)}), \bQ(\d y_i \mid y_{\pa(i)})\right)$.
\end{proposition}

Proposition~\ref{prp:opx_prop} ensures that $\cF$-compatible plans preserve the causal graph structure in conditional distributions. Here, we present the main properties of $\PFPQ$.
\begin{proposition}
\label{prp:omx_topol}
If $\bP, \bQ \in \PFX$, then $\PFPQ$ are non-empty and weakly closed.
\end{proposition}
By demonstrating the properties of definitions, we present our new OT problem, which minimizes transport costs over plans that preserve the structural equations $\cF$.
\begin{definition}[\textbf{Structural Causal OT}]
\label{def:scot}
For $p \in [1, \infty)$ and $\bP, \bQ \in \PFX \cap \cP^p(\cX)$, the structural causal Wasserstein distance is finding the minimum-cost $\cF$-compatible plans between $\bP$ and $\bQ$:
\begin{equation}
\label{eq:WFPQ}
    \WFPQ \coloneqq \left(\inf_{\pi \in \PFPQ} \left\{\expt{(x,y)\sim\pi}[c^p(x,y)]\right\}\right)^{\frac{1}{p}} 
\end{equation}
\end{definition}
The result below outlines the fundamental properties of the structural causal Wasserstein distance.
\begin{proposition}
\label{prp:fwass}
$\WF$ is a semi-metric on $\PFX$ and attains its minimum.
\end{proposition}

Now we are ready to introduce \textbf{structural causal ambiguity set} $\WFPQ$ which is defined as
\begin{equation}
\BFPD \coloneqq \left\{\bQ \in \PFX : \WFPQ \leq \delta\right\}
\end{equation}
Since our constraints involve the $\cG$-causal information, which encompasses causal order, we intuitively expect the corresponding DAS to be nested sets. This intuition is confirmed in the following proposition.
\begin{proposition}
\label{prp:inclusion}
Let $\bP,\bQ \in \PFX$, then
for different definitions of the ambiguity set, we have: 

(i) $\WFPQ \geq \WGPQ \geq \WAPQ \geq \WPQ$,

(ii)   $\BFPD \subseteq \BGPD \subseteq \BAPD \subseteq \BPD$.
\end{proposition}

Let $\Fi$ represent the zero structural equations, i.e., $f_i \equiv 0, \forall i$, implying no causal structure in the SCM. We finish this section by explaining the duality that demonstrates the correspondence between the DAS in the feature space with causal structure $\cF$ and the DAS in the exogenous space with structural equations $\Fi$. 
In the proposition below, to highlight that the cost functions differ in the two spaces, we embed the cost function in the notation of the ambiguity set.

\begin{proposition}
\label{prp:amball}
Let $\cF$ be structural equations with bijective reduced-form mapping $g$ and $\bP \in \PFX$, then
\begin{align*}
{\cB^\cF_c(\bP,\bQ)} = 
\gpush\cB^{\Fi}_{c\circ (g \times g)}(\gmpush\bP,\delta).
\end{align*}
\end{proposition}

This property plays a crucial role in designing the relaxed OT problem.


\section{Relaxed Structural Causal Optimal Transport}
One challenge in defining new variants of OT is developing efficient algorithms to compute the Wasserstein distance. In both classical and adapted OT, entropic regularization provides an efficient solution by adding an entropy penalty term to the original problem~\cite{cuturi2013sinkhorn,eckstein2024computational}.
To provide a fast computation method, we introduce a relaxed version of structural causal OT. This modification transforms the original problem into a difference-of-convex optimization problem, making it more computationally feasible. As a result, iterative algorithms like the Sinkhorn-Knopp~\cite{benamou2015iterative} can solve the problem efficiently, significantly reducing computation time and enabling the handling of large-scale problems.
\begin{definition}[\textbf{Relaxed Structural Causal OT}]
Given $\varepsilon \geq 0$ and probability measures $\bP, \bQ \in \PX$, we define the relaxed structural causal OT by solving the following optimization problem:
\begin{equation*} 
\WRPQ^p \coloneqq 
\inf_{\pi \in \PPQ} \expt{(x,y) \sim \pi}[c^p(x,y)] + \varepsilon \KL{\pi}{\pio}
\end{equation*}
where $\pio \in \PPQ$ is obtained by the following mapping: 
\begin{equation}
    \pio = \ggpush \left(\BOIN \marg_i\left(\gmgmpush \pi\right)\right)
\end{equation}
\end{definition}

The intuition behind the definition of $\pio$ is straightforward: it acts as a projection $\pi$ onto the space $\PFPQ$. Thus, if $\pi \in \PFPQ$, then $\pio = \pi$. A small penalty value $\KL{\pi}{\pio}$ indicates that $\pi$ is close to the set $\PFPQ$.

The relaxed version simplifies the problem by shifting the search for an optimal solution from the constrained plans $\PFPQ$ to the simpler space $\PPQ$, while preserving the $\cF$ structure via a regularizer. The following proposition shows that the relaxed version converges to the structural causal OT as $\epsilon$ approaches infinity. This result guarantees the effectiveness of the relaxed solution in finding the structural causal OT.

\begin{proposition}
\label{prp:converge}
Let $\pi_\varepsilon$ be the minimizer of $\WRPQ$ then:

(i) when $\epsilon \rightarrow \infty$ then $\WRPQ \rightarrow \WFPQ$ and every cluster point in the set $\{\pi_\varepsilon\}$ is the optimal solution of $\WFPQ$.

(ii) when $\epsilon \rightarrow 0$ then $\WRPQ \rightarrow \WPQ$ and every cluster point of the set $\{\pi_\varepsilon\}$ is the optimal solution of $\WPQ$.
\end{proposition}
Hopefully, not only does $\BRPD$ converge to $\BFPD$ as $\varepsilon \to \infty$, but $\BFPD$ is also always a subset of $\BRPD$, aiding in efficiently estimating this set from above. The proposition below formalizes this result.
\begin{proposition}
\label{prp:relaxprop}
For relaxed structural causal OT and $\bP,\bQ \in \PFX$ we have: 

\noindent(i) $\WRPQ$ attains its minimum.

\noindent(ii) $\WFPQ \geq \WRPQ \geq \WPQ$,

\noindent(iii) $\BFPD \subseteq \BRPD \subseteq \BPD$.
\end{proposition}

The duality between the relaxed Wasserstein distance in feature space with structural equations $\cF$ and in exogenous space with structural equations $\Fi$ is key to designing efficient computational methods for determining structural causal distance.

\begin{proposition}
    \label{pro:rel2exo}
    For $\bP,\bQ \in \PFX$ we have:
    \begin{equation}
        W^{\cF_\varepsilon}_c = W^{\cF_\varepsilon^0}_{c\circ (g \times g)}(\gmpush \bP,\gmpush \bQ)
    \end{equation}
\end{proposition}

Now we are ready to design our algorithm. If we consider $\tP,\tQ,\tc,\tpi$ as the push-forwards of $\bP,\bQ,c$, and $\pi$ by $\gm$, then by Prop.~\ref{pro:rel2exo}, we can express $W^{\cF_\varepsilon^0}_{\tc}(\tP,\tQ)$ as:

\begin{align}
\label{eq:dconvex} 
\inf_{\tpi \in \Pi(\tP,\tQ)} \left\{ \expt{(u,v)\sim \tpi}[c(u,v)]+ 
H(\tpi) - \sum_{i=1}^n H_i(\tpi) \right\},
\end{align}

Here, $H(\tpi)$ is the entropy of the plan $\tpi$, and $H_i(\tpi)$ is the entropy of the $i$-th marginal distribution $\tpi_i$ over the coordinate $(i,i+n)$. For example, if $(u_1, \dots, u_n, v_1, \dots, v_n) \sim \tpi$, then $\tpi_1$ represents the marginal distribution on the coordinates $(u_1, v_1)$. Thus, $H_i(\tpi)$ can be written as $\bE_{\tpi}[\log(\marg_i(\tpi)(u_i,v_i))]$.
Since $H(\tpi)$ and $H_i(\tpi)$ are convex functions, the optimization problem in Eq.~\ref{eq:dconvex} is a difference of convex functions. Therefore, by applying the DC algorithm (see Alg.~\ref{alg:dcalgo}), we can estimate the value of $\tpi$. In the DC algorithm, the convex term $\sum_{i} H_i(\tpi)$ is iteratively replaced by its linear approximation, converting Eq.~\ref{eq:dconvex} into a convex problem. DC algorithm implies if we define:
\begin{align*}
    G(\tpi) \in \partial(\sum H_i)(\tpi),
\end{align*}
we can reformulate the problem as a convex optimization:
\begin{align}
\label{eq:sinkhorn}
\inf_{\tpi \in \Pi(\tP,\tQ)} \left\{ \langle \tc - \varepsilon G(\tpi),\tpi\rangle + 
H(\tpi)  \right\}.
\end{align}
Eq.~\ref{eq:sinkhorn} can be solved using the Sinkhorn method for multi-marginal OT~\cite{benamou2015iterative} (see Algo.~\ref{alg:sinkhorn}) to find the minimum cost plan, since for $\tpi \in \Pi(\tP,\tQ)$, we have:
\begin{align*}
    \marg_i(\tpi) = 
    \begin{cases}
        \marg_i(\tP) & i \leq n, \\
        \marg_{i-n}(\tQ) & i > n.
    \end{cases}
\end{align*}
In the case where $P = (p_k)_{k} \in \mathbb{R}^{N}$ corresponds to feature values $(x^k)_{k}$ and $Q = (q_r)_{r} \in \mathbb{R}^{M}$ corresponds to feature values $(y^r)_r$ instead of $\bP$ and $\bQ$, we first estimate the structural equations using sample data points. After estimating the reduced-form mapping $g$, we map the sample data to the exogenous space to obtain $u = \gm(x)$ and $v = \gm(y)$. Hence, we can express
\begin{align*}
    p_k = P(u^k_1,\dots,u^k_n), \quad q_r = P(v^r_1,\dots,v^r_n).
\end{align*}
By summing over the other coordinates, we can calculate the marginal distribution $P_i$ as:
\begin{align*}
    \tP_i(u^k_i) = \sum_{u^{j_1}_1,\dots,u^{j_{i-1}}_{i-1},u^{j_{i+1}}_{i+1},\dots u^{j_n}_{n}} P(u^{j_1}_1,\dots ,u^k_i ,\dots,u^{j_n}_{n}).
\end{align*}
Similarly, we can compute the marginal $\tQ_i$. Since we know the cost function in the exogenous space by assumption~\ref{asm:space}, we can calculate the cost tensor. Then, we apply the Sinkhorn algorithm to find the tensor $\pi$ (see Alg.~\ref{alg:sinkhorn}). The above steps are summarized in Alg.~\ref{algo:dc_MMOT}.
\begin{algorithm}[!t]
  \caption{Relaxed Structural Causal Optimal Plan}
\begin{algorithmic}
\STATE {\bfseries Input:} Probability measures $P = (p_k)_{k} \in \mathbb{R}^{N}$ for feature values $(x^k)_{k}$ and $Q = (q_r)_{r} \in \mathbb{R}^{M}$ for feature values $(y^r)_r$, $\tc$ cost function over exogenous space and regularization parameter $\epsilon$.
\STATE {\bfseries Output:} Tensor $\pi \in \PPQ$.
\begin{enumerate}
    \item Estimate structural equations $\hF$ and obtain reduced-form mappings $\hat{g},\hat{g}^{-1}$.
    \item Calculate exogenous values $(u^k)_{k}$ and $(v^r)_r$ with $u^k = \hat{g}^{-1}(x^k)$, $v^r = \hat{g}^{-1}(y^r)$.
    \item Compute marginal distributions $\tilde{P}_i$ and $\tilde{Q}_i$ for $i \in [n]$.
    \item Calculate the cost tensor on the exogenous space $C=\{\tc(u,w)\}$ where $u = (u_1^{i_1},\dots, u_n^{i_n})$, $v = (v_1^{j_1},\dots, v_n^{j_n})$ and $i_k \in [N], j_k \in [M]$.
    \item While not converged:
    \begin{itemize}
        \item Gradient step: compute the gradient of the convex term $G^{(t)} = \sum_i \nabla_{\pi} H_i(\pi^{(t)})$ (see Alg.~\ref{alg:dcalgo}).
        \item Sinkhorn step: Estimate $\pi^{(t+1)}$
        \begin{equation*}
          \pi^{(t+1)} = \arg\min_{\pi \in \Pi(\tP,\tQ)} \langle C - \varepsilon G^{(t)}, P \rangle + \varepsilon H(P),
        \end{equation*}
        by the Sinkhorn Algorithm (see Alg.~\ref{alg:sinkhorn}).
    \end{itemize}
    \item Output the tensor $\pi$ corresponding to the probability values of $(\hat{g}(u),\hat{g}(w))\in \XTX$.
\end{enumerate}
\end{algorithmic}
  \label{algo:dc_MMOT}
\end{algorithm}

Since designing the relaxed structural causal OT requires estimating structural equations, we need assurance that using sample data to estimate these equations will converge to the optimal plan. The next theorem confirms this property.
\begin{theorem}[\textbf{Finite Sample Guarantee}]
\label{thm:reladest}
Let assumption~\ref{asm:space} hold and let $\cF= \{f_i\}$ represent the continuous structural equations, with $\hF= \{\hat{f}_i\}$ denoting the estimated structural equations. Suppose $\bP, \bQ \in \PFX$ have compact support. Then, for every $\epsilon > 0$, there exists a $\delta > 0$ such that if $\|f_i - \hat{f_i}\|_\infty < \delta$, then
\begin{equation*}
    \left|W^{\cF_\varepsilon}(\bP, \bQ) - W^{\hat{\cF}_\varepsilon}(\bP, \bQ)\right| \leq \epsilon,
\end{equation*}
where $\|\cdot\|_\infty$ denotes the supremum norm.
\end{theorem}
%

\section{Concentration Inequality In Presence of SCM}
Determining the ambiguity set radius (\textbf{P2}) in DRO is crucial for balancing robustness and sample sensitivity. A smaller radius increases sensitivity to noise and reduces robustness, while a larger one enhances robustness but may overlook the true distribution's behavior. The optimal choice involves estimating the magnitude of $W(\bP^N,\bP)$.

Numerous studies explore the concentration of $W(\bP^N,\bP)$. For example, \cite{fournier2015rate} provides convergence bounds, \cite{dedecker2019behavior} examines dependence conditions, and \cite{weed2019sharp} offers sharp inequalities. Below, we adapt \cite[Theorem 1]{fournier2015rate} and tailor the results to our non-metric cost function.
\begin{figure*}
    \centering
    \includegraphics[width=1\linewidth]{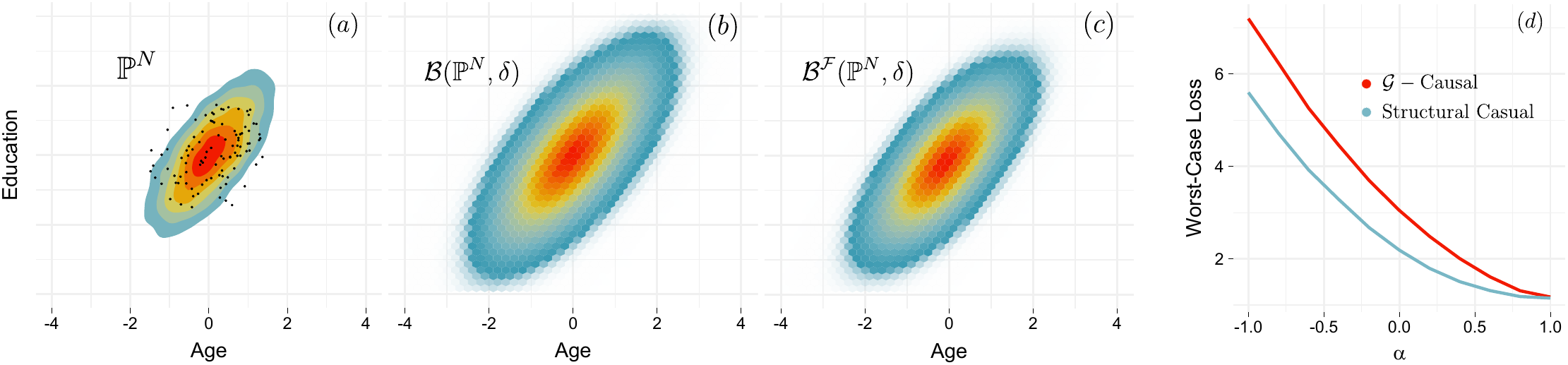}
    \caption{(a) Empirical estimation of true probability distribution for the model $\E = \A + \U_\E$ (Age and Education is normalized). (b) Ambiguity set obtained via classical OT with radius 0.5. (c) Structural causal ambiguity set with radius 0.5. (d) Comparing Worst-case losses for the structural causal and $\cG$-causal DAS with radius $\delta = 0.5$ and function $\psi(x,y) = (x - y)^2$.\\}
    \label{fig:simulation}
\end{figure*}
\begin{proposition}[\textbf{Concentration Inequality}]
\label{prp:concentration}
Let $\bP^\otimes= \bP\otimes \bP\otimes \cdots$ for the product measure on $\cX^N$, the space of all sequences of observations and
Let $\bP \in \PX$ compactly support and satisfy Assumption~\ref{asm:space}. Then for every $N \geq 1$ and any confidence level $1 - \varepsilon$ with $\varepsilon \in (0,1)$, there exists $\delta$ that holds.
\begin{align*}
\bP^ \otimes (\bP \in\cB_p(\hPN,\delta)) \ge 1 - \varepsilon,
\end{align*}
where the radius $\delta(N, \varepsilon)$ satisfies:
\begin{align}
\label{eq:upper_estimate}
\delta(N,\varepsilon) \lesssim \left(N \ln(C \varepsilon^{-1}) \right)^{-1/\max\{d,2p\}},
\end{align}
where $C$ is constant depends only to $\bP$ and $d$ dimension of feature space. Moreover, if $d \geq 2p$ and $\bP$ have a density function such that its support is compact convex, then for every $t<d$,
\begin{equation}
\label{eq:lower_estimate}
\delta(N,\varepsilon) \gtrsim N^{-1/t},
\end{equation}
\end{proposition}
Prop.~\ref{prp:concentration} shows that in high-dimensional spaces, the radius decreases slowly at a rate of $N^{-1/d}$, which is non-improvable. Thus, merely increasing the sample size offers limited improvement in approximating the true distribution or shrinking the ambiguity ball. To overcome this, we exploit the independence of components in the exogenous space, enabling a more refined ambiguity set than traditional Wasserstein sets. This approach mitigates the curse of dimensionality and ensures performance regardless of dimension $d$.

The key idea of the theorem is to leverage the causal structure to construct the empirical distribution instead of directly constructing $\bP^N$. Given samples $(x^i)_{i}$, we first derive the corresponding exogenous samples $(u^i)_{i}$ and then construct the empirical distribution $\bP_{\U_k}^N := \frac{1}{N} \sum_{i=1}^N \delta_{u_k^i}$ for each exogenous component. By independence assumption, $\hPN_\U$ is obtained as $\hPN_\U = \bP_{\U_1}^N \odo \bP_{\U_n}^N$. Finally, by mapping back to the feature space, we construct $\hPT = g_{\#}\hPN_\U$.

\begin{proposition}
\label{prp:concentration2}
Let $\bP \in \PX$ compactly support and satisfy the assumption~\ref{asm:space}. Then for every $N \geq 1$ and any confidence level $1 - \varepsilon$ with $\varepsilon \in (0,1)$, there exists $\delta$ that holds.
\begin{align*}
\bP^ \otimes (\bP \in\cB^{\cF}(\hPT,\delta)) \ge 1 - \varepsilon,
\end{align*}
where the radius $\delta(N, \varepsilon)$ satisfies:
\begin{align}
\label{eq:breack}
\delta(N,\varepsilon) \lesssim \left(N \ln(C n\varepsilon^{-1}) \right)^{-1/\max\{d^*,2p\}},
\end{align}
where $d^* = \max_{i=1}^n d_i$ and $C$ is constant depends only to $\bP$ and $d_i$. 
\end{proposition}

We conclude this section with the following corollary, which demonstrates that the convergence rate in structural causal models does not depend on the dimension of the space.

\begin{corollary}
\label{cor:breaks}
If $d_i = 1$ and $d \geq 2p + 1$, then $W(\hPN, \bP) \lesssim N^{-1/d}$, however $W^\cF(\hPT, \bP) \lesssim N^{-1/2p}$. This implies that the dependence is only on $\bP$, allowing us to break the curse of dimensionality.
\end{corollary}

\begin{table}[h!]
\centering
\begin{tabular}{lcccc}
\hline
$\psi(x,y)$ & $\delta = 0.1$ & $\delta = 0.2$ & $\delta = 0.3$ & $\delta = 0.4$ \\
\hline
$\abs{x-y}$    & 2.70 &  6.35 &  9.95 & 13.6 \\
$(x-y)^2$   &  5.54 & 12.9 &  22.0 &  29.6\\
$\abs{x+y}$    &  0.722 & 0.730 & 0.830 & 1.05 \\
$(x+y)^2$   &  1.10 &  1.32 &  1.68 &  1.86\\
$x^2+y^2$  &   2.46 &  5.15 &  7.53 & 11.2\\
\hline
\end{tabular}
\caption{Shows the additional percentage of worst-case loss $\sup_{\bQ \in \BGPD}\bE[\psi]$ relative to $\sup_{\bQ \in \BFPD}\bE[\psi]$ for different ambiguity radius $\delta$ values and functions. A positive value indicates the fact $\BFPD \subseteq \BGPD$.\\}
\label{tab:worst_loss}
\end{table}

\section{Experimental Evaluation}
\label{sec:experimental}
As our method is a novel variant of OT, it is applicable in any scenario where OT or Wasserstein distance has been previously employed, particularly when the data model is derived from causal structures. Fields such as transfer learning, reinforcement learning, algorithmic fairness, generative adversarial networks, and clustering (see additional applications in \cite{montesuma2023recent,khamis2024scalable}) could benefit from our approach. Therefore, a comprehensive numerical demonstration of its applications requires further independent and follow-up work.

In this section, we demonstrate that even with the simplest causal structures in data, different OT variants can produce varying results. Consider a super simple model involving two demographic variables, Age ($\A$) and Education ($\E$), which are common features in real datasets with a known causal relationship. We model this using the simple linear SCM: $\A := \U_\A; \E := \alpha \A + \U_\E$, where $\U_\E$ and $\U_\A$ are standard normal distributions (as we normalize age and education in our data). We simulate this model for varying $\alpha \in [-1,1]$ and compute the classical, structural causal and $\cG$-causal DAS for different radii $\delta \in \{0.1, 0.2, 0.3, 0.4\}$ to illustrate the differences between OT variants.
To quantify the difference of ambiguity sets, we compute the worst-case loss (Eq.~\ref{eq:worstloss}) for the functions \(\psi(x,y) = \abs{x-y}\), \((x-y)^2\), \(\abs{x+y}\), \((x+y)^2\), and \(x^2+y^2\).

To compute structural causal OT, we use Prop.~\ref{prp:amball}, which reformulates the model into a CO-OT~\cite{tran2021factored} problem. The structural causal distance is then calculated using the COOT Python package available on ~\cite{coot2023}. We generated 10,000 distributions to explore the structural causal ambiguity set.

We have a challenge in computing the $\cG$-causal distance due to the lack of direct computational methods. To overcome this, we randomly generated 4-dimensional Gaussian plans that preserve the $\cG$-causal structure, producing 10,000 distributions as samples for the $\cG$-causal DAS. For generating the classical OT DAS, tools like~\cite{namkoonglab_dro} are useful in computing the Wasserstein distance.

In Fig.~\ref{fig:simulation}(a), the empirical density is displayed. We compare our DAS with classical OT ambiguity sets by generating 1000 points from each probability measure within the DAS, aggregating the points, and plotting a heatmap. In part (b), the classical OT ambiguity set is depicted, which is larger than the structural causal DAS. Unlike classical OT, which expands in all directions disregarding causal structure, the structural causal DAS maintains causal relations. 

As shown in Table~\ref{tab:worst_loss}, the worst-case loss is consistently lower for the structural causal DAS compared to the $\cG$-causal DAS (supports Prop.\ref{prp:inclusion}). Notably, in scenarios like \((x-y)^2\), as illustrated in Fig.~\ref{fig:simulation}(d), the loss difference is significant because the $\cG$-causal DAS does not maintain causal links as effectively as the structural causal DAS. This alignment in structural causal DAS reduces loss when designing the ambiguity set around causal structural equations.
\section{Discussion and Limitations}

The main focus of this work is to establish a theoretical framework for a new variant of OT that incorporates not only the causal graph but also the magnitude of relationships between features. We address key aspects (\textbf{P1} and \textbf{P1}) of designing the new DAS and demonstrate its advantages compared to previous definitions. In the numerical section, we illustrate the impact of our method, even with the simplest causal structure.

To demonstrate the advantages of our method, further independent work focusing on real-world applications, including transfer learning, algorithmic fairness, GANs, etc., is essential to complete the theoretical aspects of our research.

To enhance our method for real-world problems, it is essential to establish a strong duality theorem to convert the DRO problem into a more computationally tractable form, which is a focus of our future work. We also aim to extend these results to general SCM models and relax our assumptions.

\clearpage
\bibliography{aaai25}

\clearpage
\appendix
\section{Appendix}
\label{sec:appendix}
\noindent\textbf{Notation.}
In this work, random variables are in bold (e.g., $\X$), their probability spaces in calligraphic letters (e.g., $\cX$), and instances in regular letters (e.g., $x$). Probability measures on $\cX$ are denoted by $\PX$, and individual measures by blackboard bold letters (e.g., $\bP$). The notation $f(n) \lesssim g(n)$ means there exists a constant $C$ such that $f(n) \leq Cg(n)$ for all $n$. We use $[n]$ to denote $\{1, \dots, n\}$.

\subsection{Supplementary Preliminary Knowledge}

\begin{definition}[\textbf{Pushforward Measure}]
\label{def:pushforwad}
Let $\bP\in\probSpace{\cX}$ and $f:\cX\to\mathcal{Y}$. Then, the pushforward of $\bP$ via $f$ is denoted by $\pushforward{f}\bP$, and is defined as $(\pushforward{f}\bP)(A)\coloneqq \bP(f^{-1}(\mathcal A))$, for all Borel sets $ A\subset\mathcal{Y}$.
\end{definition}

\begin{definition}[\textbf{Coupling}]
\label{def:coupling}
A \emph{coupling} between two probability measures $\mathbb{P}$ and $\mathbb{Q}$ on measurable spaces $(\cX, \mathcal{A})$ and $(\mathcal{Y}, \mathcal{B})$, respectively, is a probability measure $\pi$ on the product space $(\cX \times \mathcal{Y}, \mathcal{A} \otimes \mathcal{B})$ such that
\begin{align*}
\pi(A \times \mathcal{Y}) = \mathbb{P}(A) \quad \text{and} \quad \pi(\cX \times B) = \mathbb{Q}(B)
\end{align*}
for all $A \in \mathcal{A}$ and $B \in \mathcal{B}$.
\end{definition}

\begin{definition}[\textbf{Semi-Metric}]
\label{def:semi-metric}
A \emph{semi-metric} on a set $X$ is a function $d: X \times X \to \bR$ satisfying the following conditions for all $x, y, z \in X$:

    (i) $d(x, x) = 0$ (identity of indiscernibles),
    
    (ii) $d(x, y) = d(y, x)$ (symmetry),
    
    (iii) $d(x, y) \geq 0$ (non-negativity).

However, a semi-metric is not required to satisfy the triangle inequality, i.e., it is not necessary that $d(x, z) \leq d(x, y) + d(y, z)$ for all $x, y, z \in X$.
\end{definition}

\subsection{Supplementary Numerical Method}

%

A DC programming problem involves minimizing (or maximizing) a function that can be expressed as the difference between two convex functions. Formally, a DC programming problem is given by:
\begin{equation*}
\min_{x \in \mathbb{R}^n} \{ f(x) = g(x) - h(x) \mid x \in C \},    
\end{equation*} 
where $g(x)$ and $h(x)$ are convex functions on $\mathbb{R}^n$ and $C$ is a convex set, representing the feasible region.

\begin{algorithm}
\caption{DC Algorithm}
\begin{algorithmic}[1]
\REQUIRE Convex functions $g, h : \mathbb{R}^n \to \mathbb{R}$, initial point $x^0 \in \mathbb{R}^n$
\ENSURE Approximate solution $x^*$

\STATE Initialize $x^0$
\REPEAT
    \STATE Compute a subgradient $y^k \in \partial h(x^k)$
    \STATE Solve the convex optimization problem:
    $$
    x^{k+1} = \arg\min_{x \in \mathbb{R}^n} \{ g(x) - \langle y^k, x \rangle \}
   $$
\UNTIL{convergence i.e., until $\|x^{k+1} - x^k\|$ is below a predefined threshold.}
\RETURN $x^k$
\end{algorithmic}
\label{alg:dcalgo}
\end{algorithm}

The Entropic Multi-Marginal OT (MMOT) problem extends the classical OT problem to multiple probability measures~\cite{benamou2015iterative,tran2021factored}. Given a set of probability measures $P_1, P_2, \dots, P_m$ and a cost function $C(x_1, x_2, \dots, x_m)$, the objective is to find a joint probability measure $\Pi$ that minimizes the cost function while matching the given marginals. The Sinkhorn algorithm solves this problem iteratively by normalizing the joint probability measure at each step to match the marginals. The algorithm begins with initial potentials for each marginal and updates them iteratively until convergence, ensuring that the resulting transport plan is optimal with respect to the entropic regularization.

\begin{algorithm}
\caption{Sinkhorn Algorithm for Entropic MMOT}\label{alg:sinkhorn}
\begin{algorithmic}[1]
\STATE \textbf{Input:} Probability measures $P_1, \dots, P_m$, cost tensor $C$, regularization parameter $\epsilon$.
\STATE \textbf{Initialize:} Set initial potentials $\phi_i^{(0)}(x_i) = 1$ for all $i$.

\FOR{$n = 0, 1, 2, \dots$ until convergence}
    \FOR{each marginal $i = 1, \dots, m$}
        \STATE Update potentials:
        \begin{align*}
        &\phi_i^{(n+1)}(x_i) = 
        \\&
        \frac{P_i(x_i)}{\sum_{\substack{x_1, \dots, x_{i-1}, \\ x_{i+1}, \dots, x_m}} \exp\left(-\frac{C(x_1, \dots, x_m)}{\epsilon}\right) \prod_{j \neq i} \phi_j^{(n)}(x_j)}
        \end{align*}
    \ENDFOR
\ENDFOR

\STATE \textbf{Output:} Optimal transport plan:
\begin{align*}
&\Pi^\ast(x_1, x_2, \dots, x_m) = 
\\&
\exp\left(-\frac{C(x_1, x_2, \dots, x_m)}{\epsilon}\right) \prod_{i=1}^{m} \phi_i^\ast(x_i)
\end{align*}
\end{algorithmic}
\end{algorithm}

\begin{figure*}
    \centering
    \includegraphics[width=\linewidth]{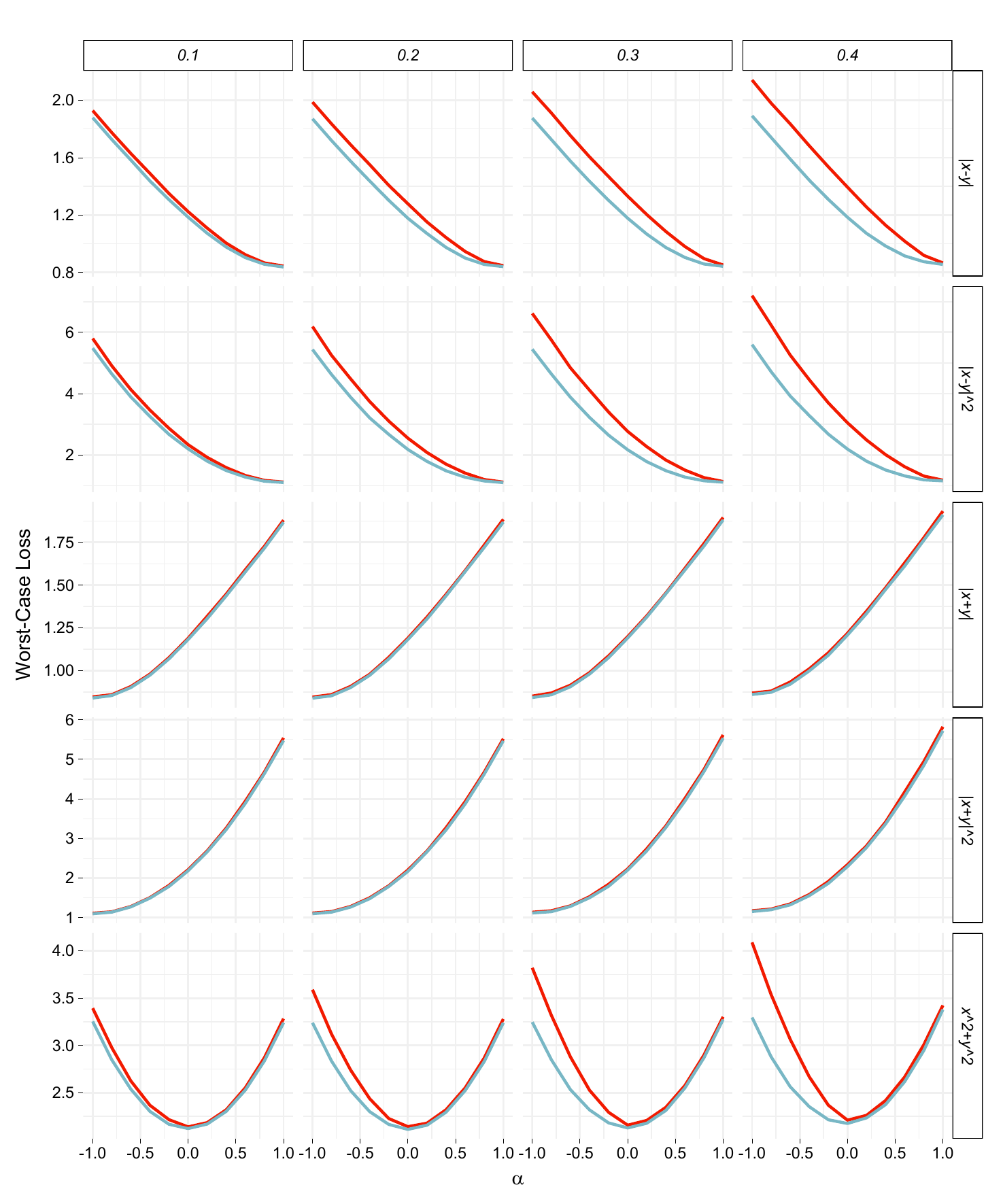}
    \caption{The worst-case loss values are shown for different levels of $\alpha$ across various functions. The red line represents the structural causal ambiguity set loss $\sup_{\bQ \in \BFPD}\bE[\psi]$, while the blue line represents the $\cG$-causal loss  $\sup_{\bQ \in \BGPD}\bE[\psi]$.}
    \label{fig:result}
\end{figure*}

\clearpage

\input{proof}
\end{document}

%% file: commands.tex
\newtheorem{definition}{Definition}
\newtheorem{theorem}{Theorem}
\newtheorem{lemma}{Lemma}

\newtheorem{proposition}{Proposition}
\newtheorem{assumption}{Assumption}

\newtheorem{corollary}{Corollary}

\newcommand{\bP}{\mathbb{P}}
\newcommand{\bQ}{\mathbb{Q}}
\newcommand{\bR}{\mathbb{R}}
\newcommand{\bE}{\mathbb{E}}

\newcommand{\bN}{\mathbb{N}}

\newcommand{\U}{\mathbf{U}}
\newcommand{\X}{\mathbf{X}}
\newcommand{\Y}{\mathbf{Y}}
\newcommand{\V}{\mathbf{V}}

\newcommand{\A}{\mathbf{A}}
\newcommand{\E}{\mathbf{E}}
\newcommand{\W}{\mathbf{W}}

\newcommand{\cX}{\mathcal{X}}
\newcommand{\cB}{\mathcal{B}}
\newcommand{\cP}{\mathcal{P}}

\newcommand{\cV}{\mathcal{V}}
\newcommand{\cY}{\mathcal{Y}}
\newcommand{\cA}{\mathcal{A}}

\newcommand{\cM}{\mathcal{M}}
\newcommand{\cU}{\mathcal{U}}

\newcommand{\cF}{\mathcal{F}}
\newcommand{\cG}{\mathcal{G}}

\newcommand{\Fi}{{\cF^0}}

\newcommand{\cC}{\mathcal{C}}
\newcommand{\tc}{\Tilde{c}}
\newcommand{\tP}{\Tilde{\bP}}
\newcommand{\tQ}{\Tilde{\bQ}}

\newcommand{\tpi}{\Tilde{\pi}}

\newcommand{\UTU}{\cU \times \cU}

\newcommand{\XTX}{\cX \times \cX}

\newcommand{\norm}[1]{\left\| #1 \right\|}

\newcommand{\expt}[1]{\underset{#1}{\mathbb{E}}}

\renewcommand{\d}{\mathrm{d}}
\newcommand{\pa}{\mathrm{pa}}

\newcommand{\ind}{{\perp\!\!\!\perp}}

\newcommand{\brak}[1]{\ensuremath{\left( #1 \right)}}

\newcommand{\abs}[1]{\ensuremath{ \left| #1 \right| }}
\newcommand{\probSpace}[1]{\mathcal{P}(#1)}
\newcommand{\pushforward}[1]{{#1}_{\#}}

\renewcommand{\d}{\mathrm{d}}

\newcommand{\diameter}

\newcommand{\supp}[1]{{\rm supp}(#1)}
\newcommand{\hPN}{\hat{\mathbb{P}}^N}

\newcommand{\hPT}{\hat{\mathbb{P}}^N_{\otimes}}
\newcommand{\odo}{\otimes \dots \otimes}
\newcommand{\gms}{g^{-1}_{\#}}

\newcommand{\marg}{\operatorname{Marg}}
\newcommand{\inv}{\operatorname{inv}}

\newcommand{\ggpush}{(g\times g)_{\#}}
\newcommand{\gpush}{g_{\#}}

\newcommand{\gmgmpush}{(g^{-1}\times g^{-1})_{\#}}
\newcommand{\gmpush}{g^{-1}_{\#}}

\newcommand{\gm}{g^{-1}}
\newcommand{\bdelta}{\bm{\delta}}
\newcommand{\pio}{\pi_{\mathsmaller\otimes}}
\newcommand{\bRP}{\bR^{\mathsmaller +}}
\newcommand{\BOIN}{\bigotimes_{i=1}^n}
\newcommand{\so}{\mathsmaller\otimes}

\newcommand{\KL}[2]{D_{\text{KL}}\left( #1\|#2 \right)}

\newcommand{\pis}{\pi^\ast}
\newcommand{\PX}{\cP(\cX)}

\newcommand{\PPQ}{\Pi(\bP,\bQ)}
\newcommand{\BPD}{\cB(\bP,\delta)}

\newcommand{\WPQ}{W(\bP,\bQ)}
\newcommand{\PPS}{\Pi(\bP, \ast)}
\newcommand{\PPD}{\Pi(\bP,\delta)}

\newcommand{\PFX}{\cP^{\mathsmaller \cF}(\cX)}

\newcommand{\PFPQ}{\Pi^{\mathsmaller \cF}(\bP,\bQ)}
\newcommand{\PFiPQ}{\Pi^{\mathsmaller \Fi}(\gmpush\bP,\gmpush\bQ)}
\newcommand{\PFPS}{\Pi^{\mathsmaller \cF}(\bP,\ast)}
\newcommand{\PFiPS}{\Pi^{\mathsmaller \Fi}(\gmpush\bP,\ast)}

\newcommand{\hF}{\hat{\cF}}

\newcommand{\FS}{\{f_i\}_{i=1}^n}

\newcommand{\WF}{W^{\mathsmaller \cF}}

\newcommand{\WFPQ}{W^{\mathsmaller \cF}(\bP,\bQ)}
\newcommand{\WFpPQ}{W^{\mathsmaller \cF}_{c}(\bP,\bQ)}

\newcommand{\BFPD}{\cB^{\mathsmaller \cF}(\bP,\delta)}

\newcommand{\PAPQ}{\Pi^{\mathsmaller \cA}(\bP,\bQ)}

\newcommand{\WAPQ}{W^{\mathsmaller \cA}(\bP,\bQ)}

\newcommand{\BAPD}{\cB^{\mathsmaller \cA}(\bP,\delta)}

\newcommand{\PGX}{\cP^{\mathsmaller \cG}(\cX)}

\newcommand{\PGPQ}{\Pi^{\mathsmaller \cG}(\bP,\bQ)}

\newcommand{\WGPQ}{W^{\mathsmaller \cG}(\bP,\bQ)}

\newcommand{\BGPD}{\cB^{\mathsmaller \cG}(\bP,\delta)}

\newcommand{\PGPS}{\Pi^{\mathsmaller \cG}(\bP,\ast)}

\newcommand{\WRPQ}{W^{\mathsmaller \cF_{\varepsilon}}(\bP,\bQ)}

\newcommand{\BRPD}{\cB^{\mathsmaller \cF_{\varepsilon}}(\bP,\delta)}

\newcommand{\PCX}{\cP^{\mathsmaller \cC}(\cX)}

\newcommand{\PCPQ}{\Pi^{\mathsmaller \cC}(\bP,\bQ)}

\newcommand{\WCPQ}{W^{\mathsmaller \cC}(\bP,\bQ)}

%% file: proof.tex
\section{Supplementary: Proof Section}

\subsection{Proof of Proposition~\ref{prp:SCP}}

(i) Let $\bP \in \PFX$. By definition~\ref{def:Fcomp}, we can write:
\begin{align*}
    g^{-1}_{\#}\bP = \tP \text{ and } \tP = \BOIN \tP_i, \quad \tP_i \in \cP(\cU_i).
\end{align*}
Let $\U_i$ be the random variables in the space $\cU_i$. From the above equation, it follows that $\U_i$ are mutually independent. Now, define the SCM with the structural equations $\cF = \FS$. Define $X_1 = f_1(\U_1)$, because $\cF$ is ordered with respect to the causal graph, meaning $\pa(1) = \emptyset$. By induction, we can define the random variable $\X_i = f_i(\X_{\pa(i)}, \U_i)$, because $\X_{\pa(i)}$ was defined earlier. Now we have a new SCM $\{\X_i\}$ that satisfies $\X_i = f_i(\X_{\pa(i)}, \U_i)$. 

To complete the proof, it suffices to show that $\X \sim \bP$. Since the reduced-form mapping $g$ depends only on the structural equations $\cF$, we have $\X \sim g_{\#} \tP$. Since $g$ is invertible, we have $\X \sim g_{\#} \tP = g_{\#} (g^{-1}_{\#} \bP) = \bP$.

(ii) By the disintegration theorem~\citep[Theorem 3.4]{kallenberg1997foundations}, we have:
\begin{align*}
\mu(\d x_1, \dots, \d x_n) = \bigotimes_{i=1}^n \mu \brak{\d x_i \mid x_1, \dots, x_{i-1}}.
\end{align*}
Since $\U_i$ are mutually independent, considering the ordered index and the SCM representation from part (i), it follows that $\U_i \ind \X_{1:i-1}$. Given that $\X_i = f_i(\X_{\pa(i)}, \U_i)$, we obtain:
\begin{align*}
\mu \brak{\d x_i \mid x_1, \dots, x_{i-1}} = \mu \brak{\d x_i \mid x_{\pa(i)}},    
\end{align*}
which completes the proof.\hfill \qed


\subsection{Proof of Proposition~\ref{prp:opx_prop}}
(i) Since $\bP, \bQ \in \PFX$, Proposition~\ref{prp:SCP} states that there exist mutually independent random variables $\{\U_i\}_{i=1}^n$ and $\{\U'_i\}_{i=1}^n$ such that for all $i \in [n]$ we have:
\begin{align}
\label{eq:form}
    \X_i = f_i(\X_{\pa(i)}, \U_i), \quad  \Y_i = f_i(\Y_{\pa(i)}, \U'_i).
\end{align}

By the BGM assumption, we can write $\X = g(\U)$. From equations~\ref{eq:form}, there exists a function $l_i:\cX_{i} \times \cX_{\pa(i)} \to \cU_i$ such that $\U_i = l_i(\X_i, \X_{\pa(i)})$. 

Utilizing \citep[Lemma 3.4]{kallenberg1997foundations}, we can express $\U'_i = k_i(\U_i, \V_i)$, where $k_i$ is a measurable map and $\V_i$ is an $\bR^{d_i}$-valued random variable that is independent of $\U_i$. By substituting this into Eq.~\ref{eq:form}, we get:
\begin{equation}
\label{eq:form2}
    \begin{aligned}
    \Y_i &= f_i(\Y_{\pa(i)}, \U'_i) = f_i(\Y_{\pa(i)}, k_i(\U_i, \V_i))  
    \\&= f_i(\Y_{\pa(i)}, l_i(h_i(\X_i, \X_{\pa(i)}), \V_i)) 
    \\&= h_i(\X_{i}, \X_{\pa(i)}, \Y_{\pa(i)}, \V_i).
\end{aligned}
\end{equation}

Similarly, we can express $\X$ in terms of $\Y$ using equations
\begin{equation}
    \label{eq:form3}
    \X_i = h'_i(\Y_{i}, \Y_{\pa(i)}, \X_{\pa(i)}, \W_i)
\end{equation}
(ii) Let $(X,Y)\sim \pi$. Similarly to the proof of Proposition \ref{prp:SCP}, by using the disintegration theorem, $\pi$ can be factored as:
\begin{align*}
\pi(\d x_1, \d y_1, \dots, \d x_n, \d y_n) = \bigotimes_{i=1}^n \pi \brak{\d x_i, \d y_i \mid x_{1:i-1}, y_{1:i-1}}
\end{align*}

From the results of the first part of the proposition, we have:
\begin{align*}
\pi(\d x_1, \d y_1, \dots, \d x_n, \d y_n) = \bigotimes_{i=1}^n \pi \brak{\d x_i, \d y_i \mid x_{\pa(i)}, y_{\pa(i)}}
\end{align*}

Using equations~\ref{eq:form} and~\ref{eq:form2}, we get:
\begin{align*}
    & X_i \ind_{X_{\pa(i)}} (X_{1:i-1}, Y_{1:i-1}) \text{ and }
    \\&
    Y_i \ind_{X_i, X_{\pa(i)}, Y_{\pa(i)}} (X_{1:i}, Y_{1:i-1}).
\end{align*}

Therefore, $\pi$-almost surely, the following two equations hold:
\begin{align*}
& \pi(d x_i \mid x_{1:i-1}, y_{1:i-1}) = \mu(d x_i \mid x_{\pa(i)}),
 \\&
\pi \brak{\d y_i \mid x_{1:i}, y_{1:i-1}} = \pi \brak{\d y_i \mid x_i, x_{\pa(i)}, y_{\pa(i)}}.
\end{align*}

The last equations show that:
\begin{align*}
\pi(\d x_i, \d y_i \mid x_{\pa(i)}, &y_{\pa(i)}) \in 
\\& \Pi\left(\bP(\d x_i \mid x_{\pa(i)}), \bQ(\d y_i \mid y_{\pa(i)})\right).  
\end{align*} \hfill \qed

\subsection{Proof of Proposition~\ref{prp:omx_topol}}
(i) Consider the structural equations $\cF = \FS$. By the Kolmogorov extension theorem~\citep[Theorem 11.4]{kallenberg1997foundations}, there exist mutually independent random variables $\{\U_i\}_{i=1}^n$ over $\cU_i$. Let $\X_i = f_i(\X_{\pa(i)}, \U_i)$ and define the probability measure $\bP$ such that $\X \sim \bP$. By definition, $\bP \in \PFX$, demonstrating that $\PFX$ is non-empty. 

(ii) Since $\bP, \bQ \in \PFX$, by Proposition~\ref{prp:SCP}, there exist random variables $\X, \Y$ and independent random variables $\{\U_i\}_{i=1}^n$ and $\{\U'_i\}_{i=1}^n$ such that $\X \sim \bP$ and $\Y \sim \bQ$ where:
\begin{align}
\X_i = f_i(\X_{\pa(i)}, \U_i) \quad \text{and} \quad \Y_i = f_i(\Y_{\pa(i)}, \U'_i).
\end{align}

By classical results in OT, for each $i$, the set $\Pi(\U_i, \U'_i)$ is non-empty. Let $\tpi \in \Pi(\U_i, \U'_i)$ and $\tpi = \bigotimes_{i=1}^n \tpi_i$. Then the pushforward plan $\pi = g_{\#} \tpi$ is in $\PFPQ$, showing it is non-empty.

To show closeness, let $\{\pi^k\}_{k=1}^{\infty}$ be a sequence of $\cF$-compatible plans. Let $\tpi^k$ be the pushforward of $\pi^k$ on the exogenous space. By definition, we have:
\begin{align*}
    \tpi^k = \bigotimes_{i=1}^n \tpi_i^k, \quad \tpi_i^k \in \Pi(\tP_i, \tQ_i),
\end{align*}
where $\tP_i = \marg_i(g_{\#} \bP)$ and $\tQ_i = \marg_i(g_{\#} \bQ)$. Since the set of couplings $\Pi(\tP_i, \tQ_i)$ is closed, we have $\tpi_i^k \to \tpi_i \in \Pi(\tP_i, \tQ_i)$. Therefore, it is sufficient to prove that $\bigotimes_{i=1}^n \tpi_i^k \rightarrow \bigotimes_{i=1}^n \tpi_i$ weakly.

Since $\tpi_i^k \to \tpi_i$, for each bounded continuous function $h_i: \cU_i \times \cU_i$ we have:
\begin{align*}
    \int_{\cU_i \times \cU_i} h_i(u_i, u'_i) \, d\tpi_i^k(u_i, u'_i) \to \int_{\cU_i \times \cU_i} h_i(u_i, u'_i) \, d\tpi_i(u_i, u'_i).
\end{align*}

Therefore, for functions of the form 
\begin{align*}
h(u_1, \dots, u_n, u'_1, \dots, u'_n) = \sum_{k=1}^m \prod_{i=1}^n h_i^k(u_i, u'_i),
\end{align*}
we have:
\begin{align*}
    \int_{\cU \times \cU} h(u, u') \, d\bigotimes_{i=1}^n \tpi_i^k(u_i, u'_i) \to \int_{\cU \times \cU} h(u, u') \, d\bigotimes_{i=1}^n \tpi_i(u_i, u'_i).
\end{align*}

Consider a bounded continuous function $h$ on $\cU \times \cU$. By the Stone-Weierstrass theorem~\citep[\S 5.7]{rudin1976principles}, we can approximate $h$ uniformly by finite sums of the form $\sum_{k=1}^m \prod_{i=1}^n h_i^k(u_i, u'_i)$, where $h_i^k$ are continuous functions on $\cU_i$. It follows that:
\begin{align*}
\bigotimes_{i=1}^n \tpi_i^k(u_i, u'_i) \to \bigotimes_{i=1}^n \tpi_i(u_i, u'_i).
\end{align*}

Therefore, the plan $\pi = g_{\#} \tpi$ is in $\PFPQ$, completing the proof. \hfill \qed

\begin{lemma}
    \label{lem:fimes}
    Let $\cF$ be a structural equation with reduced-form mapping $g$,
    if $\bP \in \PFX$, then $\gmpush \bP \in \cP^\Fi(\cU)$.
\end{lemma}
\paragraph{Proof.}
By definition $\bP \in \PFX$ means there exists $\tP_i \in \cP(\cU_i)$ such that   $\gmpush \bP = \BOIN \tP_i$. Let define $\tP = \BOIN \tP_i$. Since in $\Fi$ there is no relation between variables therefore $\tP \in \cP^\Fi(\cU)$ and it completes the proof.\hfill \qed
\begin{lemma}
\label{lem:plans}
Let $\bP, \bQ\in \PFX$, and $g$ be corresponding reduced-form mapping $g$. Then,
\begin{equation*}
\gmgmpush  \PFPQ = \Pi^\Fi(\gmpush \bP, \gmpush \bQ).
\end{equation*}
\end{lemma}
\paragraph{Proof.}
By assumption $\bP, \bQ\in \PFX$ and $\pi \in \PFPQ$. By definition 
$$\gmpush \bP = \BOIN \tP_i, \text{ and } \gmpush \bQ = \BOIN \tQ_i$$ 
and
$$\gmgmpush\pi= \bigotimes_{i=1}^n \tpi_i \text{ where }\tpi_i \in \Pi\left(\tP_i, \tQ_i\right).$$
By lemma~\ref{lem:fimes} we have $\gmpush \bP, \gmpush \bQ \in \cP^\Fi(\cU)$.
By definition of $\Fi$ which does not have any relation between variables, it results that $\gmgmpush\pi \in \Pi^\Fi(\gmpush \bP, \gmpush \bQ)$ and results $\gmgmpush  \PFPQ \subset \Pi^\Fi(\gmpush \bP, \gmpush \bQ)$

To show the inverse inclusion, let $\tpi \in \Pi^\Fi(\gmpush \bP, \gmpush \bQ)$, then
$\tpi = \BOIN \tpi_i$ where $\tpi_i \in \Pi(\tP_i,\tQ_i)$. If define $\pi = \ggpush \tpi$ then by definition $\pi \in \PFPQ$, therefore $\Pi^\Fi(\gmpush \bP, \gmpush \bQ) \subset \gmgmpush  \PFPQ$ and completes the proof. \hfill\qed


\begin{lemma}
\label{lem:ambiguity}
Let $\cF$ be structural equations with reduced-form mapping $g$ , then
$$W^\cF_{c}(\bP,\bQ) = W^\Fi_{c\circ (g\times g)}(\gmpush\bP,\gmpush\bQ).$$
\end{lemma}

\paragraph{Proof.}
Since $g$ is bijective with the inverse $f^{-1}$, by Lemma \ref{lem:plans}, for any coupling $\PFPQ$ on $\XTX$, $(\gm \times \gm)_\# \PFPQ$ is a coupling of $\Pi^\Fi(\gmpush \bP, \gmpush \bQ)$ on $\UTU$.
Consider the Wasserstein distance $W^\cF_{c}(\bP, \bQ)$:
$$W^\cF_{c}(\bP, \bQ) = \left( \inf_{\pi \in \PFPQ} \int_{\cX \times \cX} c(x, y)^p \, d\pi(x, y) \right)^{1/p}.$$
The cost function $\tc = c \circ (g \times g)$ on $\UTU$ is given by $\tc(u_1,u_2) =  c \circ (g \times g)(u_1, u_2) = c(g(u_1), g(u_2))$.
let $\pi \in \PFPQ$, the optimal solution in $\WFPQ$, then $\tpi = \gmgmpush \pi \in \Pi^\Fi(\gmpush \bP, \gmpush \bQ)$. By changing the variables in the integral we have:
\begin{align*}
    \int_{\XTX} c(x_1, x_2)^p \, \d\pi(x_1, x_2) = \int_{\UTU} \tc(u_1, u_2)^p \, \d\tpi(u_1, u_2)
\end{align*}
Therefore,
$$W^\cF_{c}(\bP,\bQ) \geq W^\Fi_{c\circ (g\times g)}(\gmpush\bP,\gmpush\bQ)$$
since the $g$ is invertible by the same reasoning 
$$W^\cF_{c}(\bP,\bQ) \leq W^\Fi_{c\circ (g\times g)}(\gmpush\bP,\gmpush\bQ)$$
we can prove the equality.
\hfill\qed

\subsection{Proof of Proposition~\ref{prp:fwass}.}
obviously $\WF$ is positive. 
To show symmetric property, let $\pi^* \in \PFPQ$ such that $\WFPQ = \expt{(x,y)\sim \pi^*}[c(x,y)]$. By definition~\ref{def:fplan} and lemma~\ref{lem:plans}, there exists plans $\tpi^* = \gmgmpush \pi^*$ such that $\tpi^* = \BOIN \tpi_i^*$ such that for it  
$$\WFPQ = \expt{(u,u')\sim \tpi^*}[c(g(u),g(u'))] = \expt{(u,u')\sim \tpi^*}[\tc(u,u')]$$
By assumption~\ref{asm:space} $\tc$ is symmetric. Consider the inverse map $\inv: (u,u') \rightarrow (u',u)$. For the push-forward measure $\inv_{\#}\tpi^*$ we have 
\begin{align*}
    \inv_{\#}\tpi^* = \BOIN \inv_{\#}\tpi_i^*
\end{align*}
and $\WFPQ = \expt{(u',u)\sim \inv_{\#}\tpi^*}[\tc(u',u)]$, so by lemma~\ref{lem:plans} and definition~\ref{def:fplan}  $\ggpush(\inv_{\#}\tpi^*)\in \Pi^\cF(\bQ,\bP)$, therefore we have
$\WFPQ\leq \WF(\bQ,\bP)$. Similarly, we have $\WF(\bQ,\bP) \leq \WFPQ$ and results in the symmetric property.

Now we show that the $\WFPQ$ is finite. We first show if $\bP \in \cP_{c}(\cX)$ then $\gmpush \bP \in \cP_{\tc}(\cU)$. To show it is sufficient write the definition. $\bP \in \cP_{c}(\cX)$ so there exists $x_0$ such that
\begin{align*} 
\int_{\cX} c(x, x_0)^p \ \bP(\d x) < \infty \Rightarrow \int_{\cU} c(g(u), g(u_0))^p \ \tP(\d u) < \infty
\end{align*}
So $\tP\in \cP_{\tc}(\cU)$. By using Lemma~\ref{lem:ambiguity} we can write

\begin{align*}
&\WFpPQ  = W^{\cF_0}_{\tc}(\tP,\tQ) \le \int_{\UTU} 
(\tc(u,u'))^p \tpi(\d u, \d u') 
\\&
\le \int_{\UTU} 
(\tc(u,u_0) + \tc(u',u_0) )^p \tpi(\d u, \d u') 
\\&
\le 2^{p-1} \brak{ \int_{\cU} 
(\tc(u,u_0))^p \ \tP(\d u) + \int_{\cU} 
(\tc(u, u_0))^p \ \tQ(\d u')} 
\\&
< \infty.
\end{align*}
To show that $\WFPQ$ attains its minimum. From definition there exists sequence  $\{\pi^k\}_{k=1}^{\infty}$ such that $\WFPQ = \lim_{k \to \infty} \expt{(x,y)\sim \pi^k}[c(x,y)]$, by closeness of $\PFPQ$, so $\pi^k \to \pi \in \PFPQ$ and it completes the proof. \hfill\qed

\begin{lemma}
\label{lem:inclusion}
For different definitions of the ambiguity set, we have: 

(i) $\PFX \subseteq \PGX$,

(ii) $\PFPQ \subseteq \PGPQ \subseteq \PAPQ \subseteq \PPQ$.
\end{lemma}

\paragraph{Proof.}
\textbf{(i)} 
Let $\bP \in \PFX$ then by proposition~\ref{prp:SCP}, then The measure $\bP$ can be decomposed as
\begin{equation*}	
\bP(\d x_1, \dots, \d x_n) = \bigotimes_{i=1}^n \bP \brak{\d x_i \mid x_{\pa(i)}},
\end{equation*}
therefore $\bP \in \PGX$.

\textbf{(ii)}

$\PFPQ \subseteq \PGPQ$:
Let $\pi \in \PFPQ$. Proposition~\ref{prp:opx_prop} shows that for all $i = 1, \dots, n$ and for $\pi$-almost all $(x, y) \in \cX \times \cY$:
\begin{align*}
&\pi(\d x_1, \d y_1, \dots, \d x_n, \d y_n) = \bigotimes_{i=1}^n \pi(\d x_i, \d y_i \mid x_{\pa(i)}, y_{\pa(i)}) 
\\&
\text{and} \quad \pi(\d x_i, \d y_i \mid x_{\pa(i)}, y_{\pa(i)}) \in \Pi_i, \quad \forall i \in [n],
\end{align*}
where $\Pi_i = \Pi\left(\bP(\d x_i \mid x_{\pa(i)}), \bQ(\d y_i \mid y_{\pa(i)})\right)$. Therefore, $\pi$ satisfies the definition of a $\cG$-compatible plan in Eq.~\ref{eq:gcom}, resulting in $\pi \in \PGPQ$.

$\PGPQ \subseteq \PAPQ$:

If $\pi \in \PGPQ$ by Theorem 3.4 \cite{cheridito2023optimal} we have:  

for $(\X, \Y) \sim \pi$, one has
\begin{align*} \Y_i \ind_{\X_{i}, \X_{\pa(i)}, \Y_{\pa(i)}} (\X, \Y_{1:i-1}) 
	\quad \mbox{for all } i =1,\dots, n,
	\end{align*}
the above equation results
$\pi(\d y_i \mid \d x_1, \dots, \d x_n) = \pi(\d y_i \mid \d x_{\pa(i)})$ therefore satisfies in the definition of adopted plan.\hfill\qed


\subsection{ Proof of Proposition~\ref{prp:inclusion}.}
\textbf{(i)}  
By using Lemma~\ref{lem:inclusion}, we have:
\begin{align*}
\PFPQ \subseteq \PGPQ \subseteq \PAPQ \subseteq \PPQ.
\end{align*}
Therefore,
\begin{align*}
\inf_{\pi \in \PFPQ} \bE_\pi[c(\X,\Y)] & \geq \inf_{\pi \in \PGPQ} \bE_\pi[c(\X,\Y)] \implies 
\\&
\WFPQ \geq \WGPQ.
\end{align*}
Similarly, we have $\WGPQ \geq \WAPQ \geq \WPQ$.

\textbf{(ii)}  
By definition, 
\begin{align*}
\BGPD = \left\{\marg_2(\pi): \pi \in \PFPS ; \bE_{\pi}[c(\X,\Y)] \leq \delta \right\}.
\end{align*}
Since by part (i) we have $\PFPS \subseteq \PGPS$, it follows that $\BFPD \subseteq \BGPD$. The other cases are proved by similar methods. \hfill\qed


\subsection{Proof of Proposition~\ref{prp:amball}}
By the definition of $\BGPD$, we can write:
\begin{align}
\label{eq:idea1}
    &\BGPD = 
    \\&
    \left\{\marg_2(\pi): \expt{(x_1,x_2)\sim\pi}[c(x_1,x_2)] \leq \delta,  \pi \in \PFPS \right\}.\nonumber
\end{align}

We first show that $\ggpush \PFiPQ = \PFPQ$. 
\begin{align*}
    &\pi \in \PFPQ \iff \gmgmpush \pi = \BOIN \tpi_i \iff
    \\&
    \tpi := \BOIN \tpi_i, \tpi\in \PFiPQ \iff \\&
    \pi = \ggpush \tpi, \tpi \in \PFiPQ
\end{align*}

By changing variables, for each $\pi \in \PFPS$ we have:
\begin{align*}
    \expt{(x_1,x_2)\sim\pi}[c(x_1,x_2)] = \expt{(u_1,u_2)\sim \gmpush\pi}[c(g(u_1),g(u_2))]
\end{align*}
It is easy to check that:
\begin{align*}
    \marg_2(\pi) &= \marg_2(\ggpush (\gmgmpush \pi)) 
    \\&
    = \gpush\marg_2(\gmgmpush \pi)
\end{align*}
Now by replacing previous results in the Eq.~\ref{eq:idea1} we have: 

\begin{align*}
    &\BGPD = 
    \bigg\{\gpush \marg_2(\tpi): \expt{(u_1,u_2)\sim\tpi}[c(g(u_1),g(u_2))] \leq \delta,
    \\&
    \tpi \in \PFiPS \bigg\} = \gpush\cB^{\Fi}_{c\circ (g \times g),p}(\gmpush\bP,\delta)
\end{align*}
the proof is complete by the last equation. \hfill\qed


\begin{lemma}
\label{lem:kldiv}
The Kullback-Leibler divergence $D_{KL}(\mu \| \nu)$ is lower semi-continuous on the space of probability measures.
\end{lemma}

\paragraph{Proof.}
Let $(\mu_n)$ and $(\nu_n)$ be sequences of probability measures that converge weakly to a probability measure $\nu$. We need to show that
\begin{align*}
\liminf_{n \to \infty} D_{KL}(\mu_n \| \nu_n) \geq D_{KL}(\mu \| \nu).
\end{align*}
By definition, the Kullback-Leibler divergence is given by
\begin{align*}
D_{KL}(\mu_n \| \nu_n) = \int_{\mathbb{R}^d} \log \left( \frac{\d\mu_n}{\d\nu_n} \right) \d\mu_n.
\end{align*}
Since the function $x \mapsto x \log x$ is lower semi-continuous and convex, and the integral preserves lower semi-continuity, we have
\begin{align*}
\liminf_{n \to \infty} \int_{\mathbb{R}^d} \log \left( \frac{\d\mu_n}{\d\nu_n} \right) \d\mu_n \geq \int_{\mathbb{R}^d} \log \left( \frac{\d\mu}{\d\nu} \right) \d\mu.
\end{align*}
Therefore,
\begin{align*}
\liminf_{n \to \infty} D_{KL}(\mu_n \| \nu_n) \geq D_{KL}(\mu \| \nu).
\end{align*}
This completes the proof.\hfill\qed


\begin{lemma}
\label{lem:projcon}
Let $\pi$ be a joint probability measure on $\cX^n$. The operator $\pi \to \BOIN \marg_i(\pi)$ is continuous concerning the weak topology, where $\marg_i(\pi)$ denotes the marginal distribution concerning the $i$-th coordinate.
\end{lemma}

\paragraph{Proof.}
Let $(\pi^k)$ be a sequence of joint probability measures on $\cX^n$ that converges weakly to a joint probability measure $\pi$. We need to show that
\begin{align*}
\bigotimes_{i=1}^n \marg_i(\pi^k) \xrightarrow{w} \bigotimes_{i=1}^n \marg_i(\pi).
\end{align*}

For each $i$, let $\pi^k_i$ and $\pi_i$ denote the $i$-th marginal distributions of $\pi^k$ and $\pi$, respectively. By the definition of weak convergence, $\pi^k \xrightarrow{w} \pi$ implies that for any bounded continuous function $f: \cX^n \to \mathbb{R}$,
\begin{align*}
\int_{\cX^n} f \, d\pi^k \to \int_{\cX^n} f \, d\pi.
\end{align*}

This, in turn, implies that for any bounded continuous function $g: \cX \to \mathbb{R}$,
\begin{align*}
\int_{\cX} g \, d\pi^k_i \to \int_{\cX} g \, d\pi_i \quad \text{for each } i = 1, \ldots, n.
\end{align*}

Thus, $\pi^k_i \xrightarrow{w} \pi_i$ for each $i$.
Next, consider the product measure $\bigotimes_{i=1}^n \pi^k_i$. For any bounded continuous function $h: \cX^n \to \mathbb{R}$,
\begin{align*}
&\int_{\cX^n} h \, d\left( \bigotimes_{i=1}^n \pi^k_i \right) = 
\\&
\int_{\cX} \cdots \int_{\cX} h(x_1, \ldots, x_n) \, d\pi^k_1(x_1) \cdots d\pi^k_n(x_n).
\end{align*}

Since $h$ is continuous and each $\pi^k_i \xrightarrow{w} \pi_i$, the integrals converge:
\begin{align*}
&\int_{\cX} \cdots \int_{\cX} h(x_1, \ldots, x_n) \, d\pi^k_1(x_1) \cdots d\pi^k_n(x_n) \to 
\\&
\int_{\cX} \cdots \int_{\cX} h(x_1, \ldots, x_n) \, d\pi_1(x_1) \cdots d\pi_n(x_n).
\end{align*}

Thus, $\bigotimes_{i=1}^n \pi^k_i \xrightarrow{w} \bigotimes_{i=1}^n \pi_i$.
Therefore, the operator $\pi \to \bigotimes_{i=1}^n \marg_i(\pi)$ is continuous with respect to the weak topology. \hfill\qed


\begin{lemma}
\label{lem:pushcon}
Let $g: \cX \to \mathcal{Y}$ be a continuous function. The pushforward operator $\mu \to g_{\#} \mu$ is continuous concerning the weak topology on the space of measures.
\end{lemma}

\paragraph{Proof.}
Let $(\mu_n)$ be a sequence of probability measures on $\cX$ that converges weakly to $\mu$. We need to show that $(g_{\#} \mu_n)$ converges weakly to $g_{\#} \mu$.

For any bounded continuous function $h: \mathcal{Y} \to \mathbb{R}$,
\begin{align*}
&\int_{\mathcal{Y}} h \, d(g_{\#} \mu_n) = \int_{\cX} h(g(x)) \, \d\mu_n(x) \quad \text{and}
\\&
\quad \int_{\mathcal{Y}} h \, d(g_{\#} \mu) = \int_{\cX} h(g(x)) \, \d\mu(x).
\end{align*}

Since $h \circ g$ is bounded and continuous on $\cX$, weak convergence $\mu_n \xrightarrow{w} \mu$ implies
\begin{align*}
\int_{\cX} h(g(x)) \, \d\mu_n(x) \to \int_{\cX} h(g(x)) \, \d\mu(x).
\end{align*}

Therefore,
\begin{align*}
\int_{\mathcal{Y}} h \, d(g_{\#} \mu_n) \to \int_{\mathcal{Y}} h \, d(g_{\#} \mu),
\end{align*}
showing $g_{\#} \mu_n \xrightarrow{w} g_{\#} \mu$ and establishing the continuity of the pushforward operator.\hfill\qed


\subsection{Proof of Proposition~\ref{prp:converge}.}
\textbf{(i)}: By assumption $\varepsilon \to \infty$, we can set $\varepsilon_k = k\in \bN$ so $k \to \infty$. For each $k \in \bN$, let $\pi^k$, the solution of $\WRPQ$. Part (i) of Proposition~\ref{prp:relaxprop}, guarantees the existence of $\pi^k$. Since $\{\pi^k\}_{k=1}^\infty \subseteq \PPQ$ and $\PPQ$ is compact subset of all probability measure over $\XTX$, then the set of $\{\pi^k\}_{k=1}^\infty$ has a cluster point $\pis$. Without loss of generality, we can suppose that $\pi^k \to \pis$ in the weak topology. For $\pis$, we have $\KL{\pis}{\pis_\so} = 0$. To show that by part $(ii)$ of Proposition~\ref{prp:relaxprop} we have for every $\varepsilon_k$ we have:
\begin{align*}
    \WFPQ &\geq \WRPQ \Rightarrow 
    \\&
    \expt{(x,y) \sim \pi^k}[c^p(x,y)] + k \KL{\pi^k}{\pio^k} \leq L \Rightarrow 
    \\&
     \KL{\pi^k}{\pio^k} \leq \dfrac{L}{k} \Rightarrow  
     \\&
     \liminf_{k \to \infty} \KL{\pi^k}{\pio^k} \geq \KL{\pis}{\pis_\so} \Rightarrow
     \\&
     \KL{\pis}{\pis_\so} = 0
\end{align*}
The last equation is valid because KL is l.s.c. property. We claim $\pis$ is the optimal solution for $\WFPQ$.  For $\pis$ we have:
\begin{align*}
    \pis \in \PFPQ \text{ and } \expt{(x,y) \sim \pis}[c^p(x,y)] \leq \WFPQ
\end{align*}
If the quality does not happen the $\pis$ is the optimal solution of $\WFPQ$ so by contradiction we have equality. Therefore, we show that every cluster point in $\{\pi_\varepsilon\}$ is the solution of $\WFPQ$. 

\textbf{(ii)}:

Since $\varepsilon \to 0$, we can suppose $\varepsilon = \frac{1}{k}$. Let $\pi^k$ be the corresponding optimal solution of $W^{\cF_{{\varepsilon}_k}}(\bP,\bQ)$. 
Suppose $\pi^w$ is the optimal solution for $\WFPQ$.
Let suppose the cluster point $\pis$ of $\{\pi^k\}_{k=1}^\infty$ is not optimal solution $\WPQ$. By part $(ii)$ of Proposition~\ref{prp:relaxprop} we have:
\begin{align*}
\expt{\pi^k}[&c^p(\X,\Y)] + \frac{1}{k} \KL{\pi^k}{\pio^k} \geq \WPQ^p \Rightarrow
\\&
\expt{\pi^k}[c^p(\X,\Y)] + \frac{1}{k} \KL{\pi^k}{\pio^k} \geq \expt{\pi^w}[c^p(\X,\Y)] 
\\&
\text{ and } \expt{\pi^w}[c^p(\X,\Y)] + \frac{1}{k} \KL{\pi^w}{\pio^w} \leq
\expt{\pi^k}[c^p(\X,\Y)] \\& + \frac{1}{k} \KL{\pi^k}{\pio^k} \Rightarrow
\\& 
0 \leq \expt{\pi^k}[c^p(\X,\Y)] - \expt{\pi^w}[c^p(\X,\Y)] \leq \frac{1}{k} \KL{\pi^w}{\pio^w}
\\&
\Rightarrow \expt{\pis}[c^p(\X,\Y)] = \expt{\pi^w}[c^p(\X,\Y)]
\end{align*}
so $\pis$ is the solution $\WPQ$, so by contradiction, the proof was complete.\hfill\qed


\begin{lemma}
\label{lem:kkpush}
Let $f: \cX \to \cX$ be an invertible function. Then the Kullback-Leibler divergence is invariant under the pushforward by $f$:
$$D_{KL}(\mu \| \nu) = D_{KL}(f_{\#} \mu \| f_{\#} \nu).$$
\end{lemma}

\begin{proof}
Let $p(x)$ and $q(x)$ be the densities of $\mu$ and $\nu$, respectively. The densities of the pushforward measures $f_{\#} \mu$ and $f_{\#} \nu$ are:
\begin{align*} p_{f}(y) = p(f^{-1}(y)) \left| \det \left( \frac{\partial f^{-1}}{\partial y} \right) \right|, \end{align*}
\begin{align*} q_{f}(y) = q(f^{-1}(y)) \left| \det \left( \frac{\partial f^{-1}}{\partial y} \right) \right|. \end{align*}
The KL divergence between the pushforward measures is:
\begin{align*}
\begin{aligned}
&D_{KL}(f_{\#} \mu \| f_{\#} \nu) = \int_{\mathcal{X}} p_{f}(y) \log \left( \frac{p_{f}(y)}{q_{f}(y)} \right) dy =\\
& \int_{\mathcal{X}} p(f^{-1}(y)) \left| \det \left( \frac{\partial f^{-1}}{\partial y} \right) \right| \log \left( \frac{p(f^{-1}(y)) \left| \det \left( \frac{\partial f^{-1}}{\partial y} \right) \right|}{q(f^{-1}(y)) \left| \det \left( \frac{\partial f^{-1}}{\partial y} \right) \right|} \right) dy \\
&= \int_{\mathcal{X}} p(f^{-1}(y)) \left| \det \left( \frac{\partial f^{-1}}{\partial y} \right) \right| \log \left( \frac{p(f^{-1}(y))}{q(f^{-1}(y))} \right) dy.
\end{aligned}
\end{align*}
By changing variables $x = f^{-1}(y)$, $dy = \left| \det \left( \frac{\partial f}{\partial x} \right) \right| dx$, the integral becomes:
\begin{align*}
D_{KL}(f_{\#} \mu \| f_{\#} \nu) = \int_{\mathcal{X}} p(x) \log \left( \frac{p(x)}{q(x)} \right) dx = D_{KL}(\mu \| \nu).
\end{align*}
Thus, the KL divergence is invariant under the pushforward by $f$.
\end{proof}

\subsection{Proof of Proposition~\ref{prp:relaxprop}}

\textbf{(i)}
To prove the statement, we use the fact that a lower semi-continuous (l.s.c.) function on a compact set attains its minimum. It is well known that $\PPQ$ is a compact subset in the set of all joint probability measures. Thus, we need to show that the operator $\pi \to \expt{(x,y) \sim \pi}[c^p(x,y)] + \varepsilon \KL{\pi}{\pio}$ is l.s.c..
Since the cost function $c$ is l.s.c., the function $\pi \to \expt{(x,y) \sim \pi}[c^p(x,y)]$ is also l.s.c. We now need to show that $\pi \to \KL{\pi}{\pio}$ is also l.s.c.

By Lemma~\ref{lem:projcon}, we know that the mapping $\pi \to \bigotimes_{i=1}^n \marg_i(\pi)$ is a continuous function. The lemma~\ref{lem:pushcon} also shows that if $g$ is a continuous function, then the operators $\pi \to \gmgmpush\pi$ or $\pi \to \ggpush\pi$ are also continuous. Combining these results, it can be seen that the mapping $\pi \to \pio$ is a continuous operator concerning the weak topology.

Lemma~\ref{lem:kldiv} shows that the KL operator is l.s.c., so the combination $\KL{\pi}{\pio}$ is also l.s.c., completing the proof.

\textbf{(ii)}
First, we show that $\WFPQ \geq \WRPQ$. Let $\pi$ be the optimal plan for $\WFPQ$. Since $\pi \in \PFX$ it results that $\pio = \pi$ and 
\begin{align*}
  \expt{(x,y) \sim \pi}[c^p(x,y)] + & \KL{\pi}{\pio} = \expt{(x,y) \sim \pi}[c^p(x,y)] 
  \\&
  \implies \WFPQ \geq \WRPQ
\end{align*}

To prove the next inequality, we have:
\begin{align*}
    &\expt{(x,y) \sim \pi}[c^p(x,y)] + \KL{\pi}{\pio} \geq \expt{(x,y) \sim \pi}[c^p(x,y)]
    \\&
    \WRPQ \geq \WPQ
\end{align*}

\textbf{(iii)}
\textbf{Case} $\BFPD \subseteq \BRPD$: 
Let $\bQ \in \BFPD$. Then there exists $\pi \in \PFPQ$ such that $\E_\pi[c] \leq \delta$ and $\marg_2(\pi) = \bQ$. Since $\KL{\pi}{\pio} = 0$, we have 
\begin{align*}
\expt{(x,y) \sim \pi}[c^p(x,y)] + \KL{\pi}{\pio} \geq \delta,
\end{align*}
which means $\bQ \in \BRPD$.

\textbf{Case} $\BRPD \subseteq \BPD$: This case is obvious from the definition. \hfill\qed


\begin{lemma}
\label{lem:kldec}
Let $\bP$ be a probability measure on $\cX^n$. The Kullback-Leibler divergence between $\bP$ and $\BOIN \bP_i$ is given by
\begin{align*}
\KL{\bP\ \ }{\BOIN \bP_i}= H(\bP) - \sum_{i=1}^n H(\bP_i), 
\end{align*}
where $\bP_i$ is the $i$-th marginal of $\bP$, $H(\bP_1, \ldots, \bP^N)$ is the joint entropy, and $H(\bP_i)$ is the marginal entropy.
\end{lemma}

\paragraph{Proof.}
The Kullback-Leibler divergence between $\bP$ and $\BOIN \bP_i$ is defined as
\begin{align*}
&\KL{\bP\ \ }{\BOIN \bP_i} = 
\int_{\cX^n} \log \left( \frac{\d\bP}{d(\bP_1 \otimes \cdots \otimes \bP^N)} \right) \d\bP. 
\end{align*}
By the definition of the Radon-Nikodym derivative,
\begin{align*} 
& \KL{\bP\ \ }{\BOIN \bP_i} = 
\\&
\int_{\cX^n} \log \left( \frac{p(x_1, \ldots, x_n)}{\prod_{i=1}^n p_i(x_i)} \right) p(x_1, \ldots, x_n) \, \d x_1 \cdots \d x_n, 
\end{align*}
where $p(x_1, \ldots, x_n)$ is the joint density of $\bP$ and $p_i(x_i)$ is the marginal density of $P_i$.

This can be rewritten as
\begin{align*}
&\KL{\bP\ \ }{\BOIN \bP_i} = 
\\&
\int_{\cX^n} \log(p(x_1,\ldots, x_n)) \, p(x_1, \ldots, x_n) \, \d x_1 \cdots \d x_n - 
\\&
\int_{\cX^n} \log \left( \prod_{i=1}^n p_i(x_i) \right) \, p(x_1, \ldots, x_n) \, \d x_1  \cdots \d x_n. 
\end{align*}

The first term is the negative joint entropy:
\begin{align*} 
-H&(\bP) = 
\\&
\int_{\cX^n} \log(p(x_1,\ldots, x_n)) \, p(x_1, \ldots, x_n) \, \d x_1 \cdots \d x_n
\end{align*}

The second term can be separated into the sum of the marginal entropies:
\begin{align*} 
-\sum_{i=1}^n& H(\bP_i) = 
\\&
-\sum_{i=1}^n \int_{\cX^n} \log(p_i(x_i)) \, p(x_1, \ldots, x_n) \, \d x_1 \cdots \d x_n. 
\end{align*}

Combining these, we get:
\begin{align*}
\KL{\bP\ \ }{\BOIN \bP_i} = H(\bP) - \sum_{i=1}^n H(\bP_i). 
\end{align*}

This completes the proof.
\hfill\qed

\begin{proposition}
\label{prp:entropy}
    Let $\X$, and $\Y$ be two random variables with continuous and compact support density functions $f(x)$ and $g(x)$ that are bounded. So for each $\epsilon>0$  there exists $\Delta$ for each $\delta <\Delta$ we have:
    \begin{equation*}
        \abs{H(\X+\delta \Y) - H(\X)}< \epsilon.
    \end{equation*}
\end{proposition}
\begin{proof}
    Let $g_{\delta}(y)$ be the corresponding density function for the random variable $\delta \Y$ and $f_\delta$ be the density function of $\X+\delta\Y$. By the convolution formula, we have:
    \begin{align*}
        f_\delta(x) = \int f(y)g_{\delta}(x-y)\d y 
    \end{align*}

    It is easy to see that $g_{\delta}(y) = \frac{1}{\delta}g(\frac{x}{\delta})$
    we first show that $\lim_{\delta \to 0} \norm{f-f_\delta}_1 = 0$. To prove this, we use the dominated convergence Theorem (DCT). Therefore, we need point-wise convergence, i.e.,
    \begin{align}
    \label{eq:dirac}
         \lim_{\delta \to 0} f_\delta(x) = \lim_{\delta \to 0}  \int f(y) \frac{1}{\delta} g\left(\frac{x-y}{\delta}\right) \, \d y = f(x)
    \end{align}
To do that first, we show that $g_\delta(y)$ converges weakly to the Dirac delta function $\delta(y)$ as $\delta \to 0$.
The weak convergence of $g_\delta(y)$ to $\delta(y)$ means that for every smooth and compactly supported function $\phi(y)$, we have:
\begin{align*}
\lim_{\delta \to 0} \int_{-\infty}^{\infty} g_\delta(y) \phi(y) \, dy = \phi(0)
\end{align*}
Now, evaluate the integral:
\begin{align*}
\int_{-\infty}^{\infty} g_\delta(y) \phi(y) \, dy = \int_{-\infty}^{\infty} \frac{1}{\delta} g\left(\frac{y}{\delta}\right) \phi(y) \, dy
\end{align*}
Perform a change of variables with $u = \frac{y}{\delta}$, hence $y = \delta u$ and $dy = \delta \, du$:
\begin{align*}
= \int_{-\infty}^{\infty} g(u) \phi(\delta u) \, du
\end{align*}
As $\delta \to 0$, $\phi(\delta u) \to \phi(0)$. 
Since $\phi$ is compactly supported and continuous, $|\phi(\delta z)|$ is bounded by $\|\phi\|_\infty$. then $|g(z) \phi(\delta z)| \leq \| \phi \|_\infty g(z)$, where $\|\phi\|_\infty$ is the supremum norm of $\phi$ and $g(z)$ is integrable. By the Dominated Convergence Theorem:

\begin{align*}
\lim_{\delta \to 0} \int_{-\infty}^{\infty} g(u) \phi(\delta u) \, du = \phi(0) \int_{-\infty}^{\infty} g(u) \, du = \phi(0)
\end{align*}
This establishes that $g_\delta(y)$ converges weakly to $\delta(y)$ as $\delta \to 0$.

Now since f is continuous, by definition of weakly convergent, we can write
\begin{align*}
& \lim_{\delta \to 0} f_\delta(x) = \lim_{\delta \to 0}  \int f(y) \frac{1}{\delta} g\left(\frac{x-y}{\delta}\right) \, \d y =
\\&
\int f(y) \lim_{\delta \to 0}  \frac{1}{\delta} g\left(\frac{x-y}{\delta}\right) \, \d y
= \int f(y) \delta(y-x) \d y = f(x)
\end{align*}

Since $f$ and $g$, both are compact support, then $f_\delta$ is compact support, so there exists a dominant function for $f_\delta$. So using DCT result that when $\delta \to 0$ then $\norm{f_\delta - f}_1 \to 0$

To complete the proof, Since the function $x \mapsto xlog(x)$ is continuous so it $f_\delta(x)log(f_\delta(x)) \to f(x)log(f(x))$, moreover since $f_n$ are compact support there exist dominated function for $f_\delta$, by using again DCT result $\int f_\delta(x)log(f_\delta(x)) \to \int f(x)log(f(x))$ so we have $H(\X+\delta\Y) \to H(\X)$ and it completes the proof.  
\end{proof}


\begin{lemma}
\label{lem:pushentropy}
Let $X = (X_1, \dots, X_n)$ be a random vector in $\mathbb{R}^n$ with joint probability density function $f_X(x)$, and let $g: \mathbb{R}^n \to \mathbb{R}^n$ be a bijective transformation. Define $Y = g(X)$. Then the differential entropy $H(Y)$ is given by:
\begin{align*}
H(Y) = H(X) + \mathbb{E}[\log |\det J_g(X)|],
\end{align*}
where $J_g(X)$ is the Jacobian matrix of $g$ evaluated at $X$, and $\det J_g(X)$ is its determinant.
\end{lemma}

\begin{proof}
The differential entropy of $X$ is:
\begin{align*}
H(X) = -\int_{\mathbb{R}^n} f_X(x) \log f_X(x) \, dx.
\end{align*}
For the transformation $Y = g(X)$, the density function $f_Y(y)$ is given by:
\begin{align*}
f_Y(y) = f_X(g^{-1}(y)) \cdot |\det J_{g^{-1}}(y)|.
\end{align*}
Equivalently, letting $y = g(x)$, we have:
\begin{align*}
f_Y(y) = f_X(x) \cdot |\det J_g(x)|^{-1}.
\end{align*}
The differential entropy of $Y$ is:
\begin{align*}
H(Y) = -\int_{\mathbb{R}^n} f_Y(y) \log f_Y(y) \, dy.
\end{align*}
Substituting $f_Y(y) = f_X(x) \cdot |\det J_g(x)|^{-1}$ and changing variables gives:
\begin{align*}
\int_{\mathbb{R}^n} f_X(x) |\det J_g(x)|^{-1} \log \left( f_X(x) |\det J_g(x)|^{-1} \right) |\det J_g(x)| \, dx
\end{align*}
This simplifies to:
\begin{align*}
H(Y) = -\int_{\mathbb{R}^n} f_X(x) \left[ \log f_X(x) - \log |\det J_g(x)| \right] \, dx.
\end{align*}
Hence:
\begin{align*}
H(Y) = H(X) + \int_{\mathbb{R}^n} f_X(x) \log |\det J_g(x)| \, dx.
\end{align*}
The last term is the expected value $\mathbb{E}[\log |\det J_g(X)|]$, so we have:
\begin{align*}
H(Y) = H(X) + \mathbb{E}[\log |\det J_g(X)|].
\end{align*}
\end{proof}

\subsection{Proof of Theorem~\ref{thm:reladest}}
First to clarify the notation, $\varepsilon$ is regularize coefficient and $\epsilon$ is denoted the small value.
To prove the result, we assume that \( \bP \) and \( \bQ \in \PFX \) have continuous and differentiable density functions, and that the cost function in the exogenous space is \( l_p \). Both of these assumptions are reasonable and applicable in real-world scenarios.

Let $\pi^1$ and $\pi^2$ be the corresponding true and estimated relaxed structural causal OT. 
For $\pi^1$ and $\pi^2$ we have:

\begin{align*} 
&\expt{(x,y) \sim \pi^1}[c_1(x,y)] + \varepsilon \KL{\pi^1}{{\pi^1}_{\so_1}} \leq
\\&
\expt{(x,y) \sim \pi^2}[c_{1}(x,y)] + \varepsilon \KL{\pi^2}{{\pi^2}_{\so_1}} 
\end{align*}
and similarly
\begin{align*} 
&\expt{(x,y) \sim \pi^2}[c_2(x,y)] + \varepsilon \KL{\pi^2}{{\pi^2}_{\so_2}} \leq
\\&
\expt{(x,y) \sim \pi^1}[c_{2}(x,y)] + \varepsilon \KL{\pi^1}{{\pi^1}_{\so_2}} 
\end{align*}

By a combination of the two above equations we have:
\begin{align*}
     & \left|W^{\cF_\varepsilon}(\bP, \bQ) - W^{\hat{\cF}_\varepsilon}(\bP, \bQ)\right| \leq
     \max\bigg\{\\&
\underbrace{\left|\bE_{\pi^2}[c_{1}] + \varepsilon \KL{\pi^2}{{\pi^2}_{\so_1}} - 
\bE_{\pi^2}[c_2] + \varepsilon \KL{\pi^2}{{\pi^2}_{\so_2}}\right|}_{\mathbf{I}},
\\&
\underbrace{\left|\expt{\pi^1}[c_{2}] + \varepsilon \KL{\pi^1}{{\pi^1}_{\so_2}} - 
\expt{\pi^1}[c_1] + \varepsilon \KL{\pi^1}{{\pi^1}_{\so_1}}\right|}_{\mathbf{II}}
\bigg\}
\end{align*}
where $c_1(x,y) = \norm{g^{-1}(x)-g^{-1}(y)}_q$ and  $c_2(x,y) = \norm{\hat{g}^{-1}(x)-\hat{g}^{-1}(y)}_q$.
We show that if $\{f_i\}$ and $\{\hat{f}_i\}$ are close enough i.e. $\norm{f_i-\hat{f}_i}_\infty<\delta$, then $\max\left\{\mathbf{I},\mathbf{II}\right\}<\epsilon$.

\paragraph{Case $\mathbf{I}$:}
\begin{align*}
    \mathbf{I} \leq &\underbrace{\left|\bE_{\pi^2}[c_{1}]- 
\bE_{\pi^2}[c_2]\right|}_{a}  +
\\&
\varepsilon\underbrace{\left|\KL{\pi^2}{{\pi^2}_{\so_1}}  -\KL{\pi^2}{{\pi^2}_{\so_2}}\right|}_{b}
\end{align*}

Let $\norm{f}_\infty:=\sup_{i} \norm{f_i}_\infty$. To prove the first part (a) it can be written:
\begin{align*}
     &\bE_{\pi^2}\left[\left|\norm{g^{-1}(x)-g^{-1}(y)}_q-\norm{\hat{g}^{-1}(x)-\hat{g}^{-1}(y)}_q\right|\right] \leq
    \\&
    \bE_{\pi^2}\left[\norm{g^{-1}(x) - \hat{g}^{-1}(x)}\right]+\bE_{\pi^2}\left[\norm{g^{-1}(y) - \hat{g}^{-1}(y)}\right] \leq
    \\&
    2 \norm{g-\hat{g}}_\infty
\end{align*}

The above equation is true because by assumption, for ANM, we have $g^{-1}(x) = x-f(x)$.
Now we are trying to estimate the second part. Let $\tpi^{2,1} \coloneqq \gmgmpush\pi^2$ and $\tpi^{2,2} \coloneqq \left(\hat{g}^{-1} \times \hat{g}^{-1}\right)_{\#}\pi^2$. Let $(\U,\W) \sim \tpi^{2,1}$, and let $\tpi^{2,1}_i$ be the marginal distribution of $\tpi^{2,1}$ over variables $(\U_i, \W_i)$. Similarly, define $\tpi^{2,2}_i$ for $\tpi^{2,2}$. 
By using lemma~\ref{lem:kkpush}, and lemma~\ref{lem:kldec} we can write 
\begin{align*}
    \KL{\pi^2}{{\pi^2}_{\so_1}} =& \KL{\tpi^{2,1}}{\BOIN \tpi_i^{2,1}} 
    \\&
    = H(\tpi^{2,1})-\sum_{i=1}^n H(\tpi_i^{2,1})
\end{align*}
By the above equation, we can rewrite part (b) such that:
\begin{align*}
    &\left|\KL{\pi^2}{{\pi^2}_{\so_1}}  -\KL{\pi^2}{{\pi^2}_{\so_2}}\right| \leq
    \\&
    \abs{H(\tpi^{2,1}) - H(\tpi^{2,2})} + \abs{\sum_{i=1}^n H(\tpi_i^{2,1}) - H(\tpi_i^{2,2})} 
\end{align*}

If we $\X,\Y \sim \pi^2$, then the pushforward plan can be obtained by $g^{-1} \times g^{-1}$ and $\hat{g}^{-1} \times \hat{g}^{-1}$.
To estimate the first term of the above equation, by using the lemma~\ref{lem:pushentropy}, we can write:
\begin{align*}
    &\abs{H(\tpi^{2,1}) - H(\tpi^{2,2})}= 
    \bigg|H(\pi^2) + \mathbb{E}[\log |\det J_{g^{-1} \times g^{-1}}(X)|] - 
    \\&
    H(\pi^2) - \mathbb{E}[\log |\det J_{\hat{g}^{-1} \times \hat{g}^{-1}}(X)|]|\bigg| = 0
\end{align*}
The above equation is true because both $J_{g^{-1} \times g^{-1}}$ and $J_{\hat{g}^{-1} \times \hat{g}^{-1}}$ are lower triangular matrices with diagonal equals to one, so their determinant equals 1.

Now we try to estimate the $\abs{H(\tpi_i^{2,1}) - H(\tpi_i^{2,2})}$ term.
Let $\U \sim \tpi^{2,1}$, by definition, it can easily check the $$\tpi^{2,2} \sim( g^{-1} \times g^{-1}) \circ (\hat{g} \times \hat{g}) \U.$$
It is easy to see that since $g$ and $\hat{g}$ are both bijective map, then:
\begin{align}
\label{eq:ent1}
    \norm{g^{-1} - \hat{g}^{-1}}_\infty<\delta \Rightarrow&\norm{g^{-1}\circ\hat{g}(u) - u}<\delta
    \nonumber\\&
    \abs{(g^{-1}\circ\hat{g})_i(u) - u_i}<\delta
\end{align} 

By ANM assumption, we have $g_i(u) = u_i + h(u_1,...,u_{i-1})$ and $g^{-1}_i(x) = x_i + k(x_1,...,x_{i-1})$. The similar fact corresponds to $\hat{g}$ is also valid. Therefore we have:
\begin{align*}
 \abs{H(\tpi_i^{2,1}) - H(\tpi_i^{2,2})} = H(\U_i+ l(\U_1, \dots, \U_{i-1}))-H(\U_i)
\end{align*}
By Eq.\ref{eq:ent1} we have $\abs{l(\U_1, \dots, \U_{i-1})}< \delta$, So we have $\W:= \frac{1}{\delta}\abs{l(\U_1, \dots, \U_{i-1})}<1$. So $W$ is bounded and compact support. Therefore we can rewrite:
\begin{align*}
    H(\U_i+ l(\U_1, \dots, \U_{i-1}))-&H(\U_i) = 
    \\&
    H(\U_i+ \delta \W)-H(\U_i) \leq \epsilon
\end{align*}
The last inequality results By using the proposition~\ref{prp:entropy}. So we can show that:

\begin{align*}
    \left|\KL{\pi^2}{{\pi^2}_{\so_1}}  -\KL{\pi^2}{{\pi^2}_{\so_2}}\right| <\epsilon 
\end{align*}
By combining two parts (a) and (b) we can conclude that $\textbf{I}\leq \epsilon$.
The estimation in Part $II$ is very similar to that in Part $I$. Therefore, this completes the proof, and we have:

\begin{align*}
    \exists \Delta: \delta \leq \Delta \text{ if }&  \norm{f_i-\hat{f}_i}_\infty \leq \delta \implies
    \\&
    \left|W^{\cF_\varepsilon}(\bP, \bQ) - W^{\hat{\cF}_\varepsilon}(\bP, \bQ)\right| \leq \epsilon.
\end{align*}
\hfill\qed


\subsection{Proof of Proposition~\ref{prp:concentration}.}
Most theorems that estimate the lower and upper bounds of $W_{c}(\hat{\bP}_N, \bP)$ require the assumption that $c$ should be a norm. However, in our case, $c$ is not a norm, but its push-forward measure $\tilde{c}$ satisfies the norm condition. By using Lemma~\ref{prp:amball}, we have:
$$
\pushforward{f}{\mathcal{B}_{c}(\bP, \delta)} = \mathcal{B}_{c \circ (f^{-1} \times f^{-1})}(\pushforward{f}{\bP}, \delta).
$$
To determine $\delta$ under the assumption of a known function $g$, we map all the data to the exogenous space and estimate $\delta$ there.

\paragraph{Poof of Upper Bound.}
Since $\bP$ has compact support, there exists $\rho \in \bR^+$ such that $\supp{\bP} \subseteq [-\rho, \rho]^d$. To prove Eq.~\ref{eq:upper_estimate}, we adopt the following \textit{Concentration Inequality} from \citealp{boskos2020data}, which states:
\paragraph{Proposition.}
Consider a sequence $\{\X^i\}_{i =1}^N$ of i.i.d. $\bR^d$-valued random variables with a compactly supported law $\bP \in \PX$. For any $p \geq 1$, $N \geq 1$, and for any confidence level $1 - \varepsilon$ with $\varepsilon \in (0,1)$, there exists $\delta_N$ that it holds 
\begin{align*}
\bP^ \otimes (\bP \in\cB(\hPN,\delta_N)) \ge 1 - \varepsilon,
\end{align*}
where the function $\delta_N(\varepsilon, \rho)$is defined as:
\begin{equation}
\delta_N(\varepsilon, \rho) := 
\begin{cases}
\left( \frac{\ln(C \varepsilon^{-1})}{c} \right)^{\frac{1}{2p}} \frac{\rho}{N^{\frac{1}{2p}}}, & \text{if } p > d/2, \\
h^{-1}\left( \frac{\ln(C \varepsilon^{-1})}{cN} \right)^{\frac{1}{p}} \rho, & \text{if } p = d/2, \\
\left( \frac{\ln(C \varepsilon^{-1})}{c} \right)^{\frac{1}{d}} \frac{\rho}{N^{\frac{1}{d}}}, & \text{if } p < d/2,
\end{cases}
\end{equation}
where $h^{-1}$ being the inverse of $h(x) = \frac{x^2}{(\ln(2 + 1/x))^2}$ for $x > 0$, constants $C$ and $c$ depend only on $\bP$ and $d$. This result provides bounds on the Wasserstein distance between the empirical measure $\hPN$and the true measure $\bP$, ensuring that with high probability, this distance is within $\delta_N(\varepsilon, \rho)$.

To estimate the upper bound of the function $h^{-1}$ in Case (II), we need to examine the argument of the function carefully. By doing so, we observe that:  
\begin{align*}
    & y = \frac{x^2}{(\ln(2 + 1/x))^2} \implies y (\ln(2 + 1/x))^2 = x^2 \implies 
    \\&
    \sqrt{y} \cdot |\ln(2 + 1/x)| = |x| \implies \sqrt{y} \cdot \ln(2 + 1/x) = h(y) \implies 
    \\& 
    h(y)\geq \sqrt{y}
\end{align*}
By applying the estimation of $h^{-1}$ for case $p=d/2$ we have
$$
h^{-1}\left( \frac{\ln(C \varepsilon^{-1})}{cN} \right)^{\frac{1}{p}} \rho \leq 
\left( \frac{\ln(C \varepsilon^{-1})}{c} \right)^{\frac{1}{2p}} \frac{\rho}{N^{\frac{1}{2p}}}
$$
by summarizing three cases we can write
\begin{align*}
\delta_N(\varepsilon, \rho) = & \rho 
\left( \frac{\ln(C \varepsilon^{-1})}{c}N \right)^{-1/\max\{d,2p\}} \implies 
\\&
\delta(N,\varepsilon) \lesssim N^{-1/\max\{d,2p\}},
\end{align*}
\paragraph{Poof of Lower Bound.}
To prove Equation \ref{eq:lower_estimate}, we utilize Theorem 1 from the paper by \cite{weed2019sharp}. This theorem states:

\paragraph{Theorem.}
Let $p \in [1, \infty)$.
If $s > d^*(\mu)$, then
\begin{equation*}
\E[W(\mu, \hat \mu_n)] \lesssim n^{-1/s}\,.
\end{equation*}
If $t < d_*(\mu)$, then
\begin{equation*}
W(\mu, \hat \mu_n) \gtrsim n^{-1/t}\,.
\end{equation*}
where the $d^*$ and $d_*$ are upper and lower Wasserstein dimensions respectively.\qed

Assuming the support of $\bP$ is convex and compact, Example 12.7 in \cite{graf2000foundations} shows that $\supp{\bP}$ is a regular set. 

For $d$-dimensional sets in $\bR^d$, the $d$-dimensional Hausdorff measure $\mathcal{H}^d$ coincides with the Lebesgue measure $\lambda^d$ up to a constant factor. Specifically, there exists a constant $c_d > 0$ such that for any Lebesgue measurable set $A \subset \bR^d$,
$$ \mathcal{H}^d(A) = c_d \cdot \lambda^d(A). $$

Since $\bP$ is continuous concerning the Lebesgue measure, it is also absolutely continuous concerning the Hausdorff measure. Therefore, by using Proposition 8 from \cite{weed2019sharp}, we have:
\begin{equation*}
d_*(\mu) = d^*(\mu) = d\,.
\end{equation*}
Now, by applying the theorem, we complete the proof of Equation \ref{eq:lower_estimate}.
\qed

\begin{proposition}
\label{prp:independent}
Let $\cU = \cU_1 \otimes \cU_n$ be the space where $(\cU_i, d_i)$ is a metric space and $\cU$ is equipped with the metric
$c(u, u') = \left( \sum_{i=1}^n c_i(u_i,u'_i)^p \right)^{1/p}$.
Let $\bP = \bP_1 \odo \bP^N$ and $\bQ = \bQ_1 \otimes\dots \otimes \bQ_n$ be two probabilities in $\cP(\cU)$ that are constructed by tensor product of probabilities $\bP_i$ and $\bQ_i$ that belong to $\cP(\cU_i)$. Then we have:
$$ W_c(\bP, \bQ)^p = \sum_{i=1}^{n} W_{c_i}(\bP_i, \bQ_i)^p. $$
\end{proposition}

\subsection{Proof of Proposition~\ref{prp:independent}.}
The $\bP$-Wasserstein distance $W_c$ between the probability measures $\bP$ and $\bQ$ on $\cU$ is defined as:
$$ W_c(\bP, \bQ) = \left( \inf_{\pi \in \Pi(\bP, \bQ)} \int_{\cU \times \cU} c(u, u')^p \, d\pi(u, u') \right)^{1/p}, $$
where $\Pi(\bP, \bQ)$ is the set of all couplings of $\bP$ and $\bQ$.
Similarly, for each $i$, the $\bP$-Wasserstein distance $W_{c_i}$ between the probability measures $\bP_i$ and $\bQ_i$ on $\cU_i$ is defined as:
$$ W_{c_i}(\bP_i, \bQ_i) = \left( \inf_{\pi_i \in \Pi(\bP_i, \bQ_i)} \int_{\cU_i \times \cU_i} c_i(u_i, u'_i)^p \, d\pi_i(u_i, u'_i) \right)^{1/p}, $$
where $\Pi(\bP_i, \bQ_i)$ is the set of all couplings of $\bP_i$ and $\bQ_i$.

Consider $\pi = \pi_1 \otimes \pi_2 \odo \pi_n$ to be a coupling in $\Pi(\bP, \bQ)$, where each $\pi_i \in \Pi(\bP_i, \bQ_i)$. The distance $c(u, u')$ can be written as:
$$ c(u, u') = \left( \sum_{i=1}^n c_i(u_i, u'_i)^p \right)^{1/p}. $$
The integral of the $p-$th power of the distancconcerningto the coupling $\pi$ is:
\begin{align*}
& \int_{\cU \times \cU} c(u, u')^p \, d\pi(u, u') = 
\\&
\int_{\cU \times \cU} \left( \sum_{i=1}^n c_i(u_i, u'_i)^p \right) \, d(\pi_1 \otimes \pi_2 \odo \pi_n)(u, u')=
\\&
\sum_{i=1}^n \int_{\cU_i \times \cU_i} c_i(u_i, u'_i)^p \, d\pi_i(u_i, u'_i).     
\end{align*}
In the last equation, the properties of tensor products and integrals were used.
Taking the infimum over all possible couplings $\pi_i \in \Pi(\bP_i, \bQ_i)$ gives:
\begin{align*}
& \inf_{\pi \in \Pi(\bP, \bQ)} \int_{\cU \times \cU} c(u, u')^p \, d\pi(u, u') \leq 
\\&
\sum_{i=1}^n \inf_{\pi_i \in \Pi(\bP_i, \bQ_i)} \int_{\cU_i \times \cU_i} c_i(u_i, u'_i)^p \, d\pi_i(u_i, u'_i).     
\end{align*}

The last inequality is valid because $\Pi(\bP_1,\bQ_1)\odo \Pi(\bP^N,\bQ_n) \subset\Pi(\bP,\bQ)$.
Hence, we have:
$$ W_c(\bP, \bQ)^p \leq \sum_{i=1}^n W_{c_i}(\bP_i, \bQ_i)^p.$$

Conversely, by Kantorovich duality for transport costs $W^p(\bQ_i, \bP_i)$, we have:
\begin{align*}
&\sum_{i=1}^n W_{c_i}^p(\bQ_i, \bP_i) = 
\\&
\sum_{i=1}^n \sup_{\substack{(f_i, g_i) \in L^1(\bQ_i) \times L^1(\bP_i) \\ g_i(x_i) - f_i(y_i) \leq d_i(x_i, y_i,)^p}} \left\{ \int_{\cU_i} g_i(x_i)  d\bP_i - \int_{\cU_i} f_i(y_i)  d\bQ_i \right\} 
\\&
= \sup_{\substack{(f_i, g_i) \in L^1(\bQ_i) \times L^1(\bP_i) \\ g_i(x_i) - f_i(y_i) \leq d_i(x_i, y_i,)^p}} \sum_{i=1}^n \left\{ \int_{\cU_i} g_i(x_i)  d\bP_i - \int_{\cU_i} f_i(y_i)  d\bQ_i \right\} \\
&= \sup_{\substack{(\psi, \phi) \in L^1(\bQ) \times L^1(\bP) \\ \psi = \sum_{i=1}^n f_i \circ \text{pr}_i,  \phi = \sum_{i=1}^n g_i \circ \text{pr}_i \\ g_i(x_i) - f_i(y_i) \leq d_i(x_i, y_i,)^p}} \left\{ \int_\cU \phi(x)  d\bP - \int_\cU \psi(y)  d\bQ \right\} \\
&\leq \sup_{\substack{(\psi, \phi) \in L^1(\bQ) \times L^1(\bP) \\ \phi(x) - \psi(y) \leq \sum_{i=1}^n d_i(x_i, y_i,)^p}} \left\{ \int_\cU \phi(x)  d\bP - \int_\cU \psi(y)  d\bQ \right\}.
\end{align*}

The second equality follows from the decoupled constraints on $f_i$ and $g_i$. The final equality holds because if $\phi = \sum_{i=1}^n g_i \circ \text{pr}_i$, then $\phi \in L^1(\bP)$ (similarly for $\psi$). Therefore:
\begin{align*}
    \int_\cU \phi(x)  d\bP(x) &= \sum_{i=1}^n \int_{\cU_i} g_i(x_i)  d\bP_i(x_i)
\end{align*}
(and similarly for $\int_\cU \psi(y)  d\bQ(y)$). Since $\sum_{i=1}^n d_i(x_i, y_i,)^p = d(x, y)^p$, we get:
\begin{align*}
    & \sum_{i=1}^n W_{c_i}^p(\bQ_i, \bP_i) \leq 
    \\&
    \sup_{\substack{(\psi, \phi) \in L^1(\bQ) \times L^1(\bP) \\ \phi(x) - \psi(y) \leq \rho(x, y)^p}} \left\{ \int_\cU \phi(x)  d\bP(x) - \int_\cU \psi(y)  d\bQ(y) \right\} =
    \\& W_c^p(\bQ, \bP).
\end{align*}
The last equation leads to quality and completes the proof.\qed

\begin{corollary}
\label{cor:prwas}
By assumption~\ref{asm:space} the below equation holds:
\begin{align}
   W_c(\hPT, \bP)^p = \sum_{i=1}^{n} W_{c_i}(\hPN_i, \bP_i)^p. 
\end{align}
\end{corollary}

\subsection{Proof of Corollary~\ref{cor:prwas}.}

By assumption and using the lemma \ref{lem:ambiguity}  and proposition~\ref{prp:independent}, we can write
\begin{align*}
    W_{c}(\hPT,\bP) = W_{\tc}(\hPN_{\U},\bP_{\U}) = & \sum_{i=1}^n W_{\tc_i}(\hPN_{\U_i},\bP_{\U_i}) \leq 
    \\&
    \sum_{i=1}^n W_{c_i}(\hPN_i,\bP_i).
\end{align*}\qed

\subsection{Proof of Proposition~\ref{prp:concentration2}.}
By lemma~\ref{lem:ambiguity} it can be written:
\begin{align*}
W^\cF_{c}(\bP,\bP^N) = W^\Fi_{c\circ (g\times g)}(\tP,\tP^N).
\end{align*}
where $\tP = \gmpush\bP$ and $\tQ = \gmpush\bQ$. By using proposition~\ref{prp:independent} we can write:
\begin{align*}
    W^\Fi_{c\circ (g\times g)}(\tP,\tP^N)^p = \sum_{i=1}^n W_{\tc_i}(\tP_i, \tP^N_i)^p.
\end{align*}

Let $\epsilon$ be the given confidence level. By Proposition~\ref{prp:concentration}, with probability $1 - \frac{\epsilon}{n}$, the distance 
\begin{equation}
\label{eq:event}
W_{\tc_i}(\tP_i, \tP^N_i) \leq a_i \left(N \ln(A_i n \epsilon^{-1}) \right)^{-1/\max\{d_i, 2p\}},
\end{equation}
holds for all $i \in [n]$, where $a_i$ and $A_i$ depend only on $c_i$ and $\bP$.

Define $a^* = \max_{i=1}^n a_i$ and $A^* = \max_{i=1}^n A_i$. We then define the event 
\begin{align*}
E_i = \left\{W_{c_i}(\tP_i, \tP^N_i) > c^* \left(N \ln(C^* n \epsilon^{-1}) \right)^{-1/\max\{d^*, 2p\}}\right\}.
\end{align*}

By equation \eqref{eq:event}, the probability of the event $E_i$ is less than $\frac{\epsilon}{n}$. Therefore, the probability of the complementary event $\bigcup_{i=1}^n E_i$ is at most $\epsilon$. Consequently, by the union bound, we have
\begin{align*}
\bP^\otimes\left(\bigcup_{i=1}^n E_i\right) \leq \sum_{i=1}^n \bP^\otimes(E_i) \leq \epsilon.
\end{align*}

Thus, with probability at least $1 - \epsilon$, we have 
\begin{align*}
W_{\tc}(\tP,\tP^N)^p &= \sum_{i=1}^n W_{\tc_i}(\tP_i,\tP_i^N)^p \\
&\leq n a^* \left(N \ln(A^* n \epsilon^{-1}) \right)^{-p/\max\{d^*, 2p\}}.
\end{align*}

This implies that with probability at least $1 - \epsilon$,
\begin{align*}
\bP^\otimes (\bP \in \cB(\hPT, \delta)) \geq 1 - \epsilon,
\end{align*}
where $\delta = n a^* \left(N \ln(A^* n \epsilon^{-1}) \right)^{-1/\max\{d^*, 2p\}}$. This completes the proof.\qed

\subsection{Proof of Corollary~\ref{cor:breaks}.}

Since $d_i = 1$, by equation \eqref{eq:breack}, we have
\begin{align*}
\delta(N, \varepsilon) \leq d c^* \left(N \ln(C^* d \varepsilon^{-1}) \right)^{-1/2p}.
\end{align*}
Therefore, 
\begin{align*}
W_c(\hPT, \bP) \lesssim N^{-1/2p},
\end{align*}
which is independent of the dimension $d$.\qed